\documentclass[lettersize,journal]{IEEEtran}
\usepackage[commentColor=black,noEnd]{algpseudocodex}

\usepackage{amsmath,amsfonts}
\usepackage[commentColor=black]{algpseudocodex}
\tikzset{algpxIndentLine/.style={draw=black}}
\algrenewcommand{\alglinenumber}[1]{\bfseries\footnotesize #1}
\algrenewcommand{\textproc}{}
\algrenewcommand{\algorithmicrequire}{\textbf{Input:}}
\algrenewcommand{\algorithmicensure}{\textbf{Output:}}
\usepackage{algorithm}
\usepackage{array}
\usepackage{orcidlink}
\usepackage{textcomp}
\usepackage{stfloats}
\usepackage{url}
\usepackage{verbatim}
\usepackage{graphicx}
\usepackage{cite}
\usepackage[table]{xcolor}
\usepackage[utf8]{inputenc}
\usepackage[T1]{fontenc}
\usepackage{hyperref}
\usepackage{booktabs}
\usepackage{nicefrac}
\usepackage{microtype}
\usepackage[para,online,flushleft]{threeparttable}
\usepackage{pifont}
\usepackage{multirow}
\usepackage{amsthm}
\usepackage{subcaption}
\usepackage{refcount}

\newtheorem{assumption}{Assumption}

\newtheorem{lemma}{Lemma}
\newtheorem{theorem}{Theorem}

\usepackage[noEnd]{algpseudocodex}
\usepackage{setspace}

\begin{document}

\title{COOPO: Cyclic Offline-Online Policy Optimization Algorithm}

\author{%
    Qisai~Liu\IEEEauthorrefmark{1}\IEEEauthorrefmark{2},
    Zhanhong~Jiang\IEEEauthorrefmark{4},
    Joshua~R.~Waite\IEEEauthorrefmark{4},
    Aditya Balu\IEEEauthorrefmark{4},
    Cody Fleming\IEEEauthorrefmark{1},
    Soumik~Sarkar\IEEEauthorrefmark{1}\IEEEauthorrefmark{2}\IEEEauthorrefmark{4},
    \thanks{\IEEEauthorrefmark{1}Department of Mechanical Engineering, Iowa State University, Ames, IA 50011, USA.}%
    \thanks{\IEEEauthorrefmark{2}Department of Computer Science, Iowa State University, Ames, IA 50011, USA.}%
    \thanks{\IEEEauthorrefmark{3}Department of Industrial and Manufacturing Systems Engineering, Iowa State University, Ames, IA 50011, USA.}%
    \thanks{\IEEEauthorrefmark{4}Translational AI Center, Iowa State University, Ames, IA 50011, USA.}%
    \thanks{Corresponding author: Soumik~Sarkar (e-mail: soumiks@iastate.edu)}%
}

\markboth{Journal of \LaTeX\ Class Files,~Vol.~14, No.~8, August~2021}%
{Shell \MakeLowercase{\textit{et al.}}: A Sample Article Using IEEEtran.cls for IEEE Journals}

\maketitle

\begin{abstract}
Offline reinforcement learning struggles with distributional shift and constrained performance due to static dataset limitations, while online RL demands prohibitive environment interactions. The recent advent of hybrid offline-to-online methods bridges these paradigms but suffers from distributional shift during transitions and catastrophic forgetting of offline knowledge. We introduce \textit{COOPO} (Cyclic Offline-Online Policy Optimization), a generalized framework that repeatedly cycles between constrained offline training and online fine-tuning. Each cycle first anchors the policy to the dataset via KL-regularized advantage-weighted offline updates to minimize distributional shift and then fine-tunes it online using any policy optimization for stable exploration. Crucially, periodically returning to offline training eliminates forgetting and drift while maximizing dataset reuse. The cyclic behavior also helps reduce the online environment interactions. Theoretically, COOPO achieves better online sample efficiency, surpassing pure online RL, with guaranteed performance improvement under standard coverage assumptions. Extensive D4RL benchmarks demonstrate COOPO reduces online interactions versus state-of-the-art hybrids while improving final returns. This looped synergy sets new efficiency and performance standards for adaptive RL. 
\end{abstract}


\section{Introduction}

Reinforcement learning (RL) has demonstrated compelling performance and remarkable success in numerous fields over the past decade, such as robotics~\cite{brunke2022safe,singh2022reinforcement},
manufacturing~\cite{li2023deep}, medical imaging~\cite{zhou2021deep}, and reasoning with large language models (LLMs)~\cite{guo2025deepseek}. In widespread online RL approaches such as Proximal Policy Optimization (PPO)~\cite{schulman2017proximal} and Soft Actor-Critic (SAC)~\cite{haarnoja2018soft}, an agent learns the optimal policy through interactions with an environment. However, the requirement of exploration computationally poses the issue of prohibitive sample complexity. Fortunately, in many real-world cases, a rich abundance of logged data collected from experts is accessible to practitioners such that offline RL~\cite{prudencio2023survey} provides a promising alternative to learn polices. Without directly interacting with the environment, offline RL tends to learn a policy exclusively from a fixed dataset of pre-collected experiences. Nevertheless, due to the need to mitigate data drift during learning, offline RL usually suffers from policy suboptimality, limiting its applicability in unseen tasks. The dilemma of online and offline RL has recently motivated work~\cite{guo2023sample, song2022hybrid} to combine these two lines of research, whereby an optimal policy is achieved by an agent first learning from an offline dataset and then interacting with an environment for further improvement. While conceptually appealing, emerging offline-to-online RL methods endure \textit{catastrophic forgetting}~\cite{luo2023finetuning} that shows significant performance degradation when transitioning from offline learning to online fine-tuning. Since there is a distributional shift between offline and online data, the knowledge learned offline is overwritten during online fine-tuning.

In cyber-physical systems (CPS) such as autonomous driving, additive manufacturing, and energy grids, physical constraints and safety limitations make large-scale online exploration infeasible. Traditional reinforcement learning lacks guarantees when deployed in real-world applications due to its sample inefficiency and concerns about robustness. From a control perspective, such systems ensure stability and safety under limited data and uncertain dynamics, as well as resilience against cyber and communication disruptions that can compromise sensing or actuation. Recent studies have shown that data poisoning, unknown disturbances, and model manipulation can mislead learning-based controllers, highlighting the necessity for algorithms that remain stable under distribution shifts or adversarial perturbations.

To mitigate this issue, a recent work proposed an optimistic exploration and meta adaptation (OEMA)~\cite{guo2023sample} to stabilize online exploration and reduce the distributional shift during the offline-to-online transition, thus improving the sample efficiency.
A more straightforward method is to blend offline data with online interactions in the replay buffer by using off-policy techniques~\cite{song2022hybrid,ball2023efficient}, even without offline RL pre-training. While reducing catastrophic performance drops, these methods put the efficacy and necessity of contemporary RL fine-tuning methods into question. To validate online RL fine-tuning, a more recent work~\cite{zhou2024efficient} proposed Warm Start Reinforcement Learning (WSRL) to employ a warmup phase at the beginning of online fine-tuning by utilizing a small number of rollouts from the offline pre-trained policy, bridging the distribution mismatch. Even if WSRL does not retain offline data, it is unclear how long the warm up phase should be for different environments. Additionally, most existing methods have focused primarily on improving the performance when transitioning from offline pre-training to online fine-tuning, ignoring the sample efficiency itself in fine-tuning. Thus, a question naturally arises:
    \textit{Can we develop an efficient offline-to-online RL algorithm to mitigate catastrophic forgetting and reduce online interactions with the environment?}

\noindent\textbf{Contributions.} To answer this question, we introduce \textit{COOPO} (Cyclic Offline-Online Policy Optimization), a unified framework that resolves the critical limitations of hybrid reinforcement learning distribution drift during offline-to-online transitions and catastrophic forgetting of offline knowledge during online fine-tuning. By innovatively looping between Kullback-Liebler (KL)-regularized offline optimization and online adaptation (as shown in Figure~\ref{fig:coopo_curve}), COOPO dynamically anchors policies to the static dataset during offline phases, constraining distributional shift via advantage-weighted updates. It interleaves short online exploration bursts with periodic returns to offline training, eliminating forgetting while maximizing dataset reuse. In the meantime, it improves both learning stability and defense against data or policy drift in safety-critical CPS environments. COOPO formalizes a cyclic synergy where offline phases provide stability and online phases enable adaptation, transforming the coverage limitations of offline RL and sample inefficiency of online RL into complementary strengths. This architecture sets a new paradigm for sample-efficient and stable adaptive learning with theoretical and empirical guarantees. When there is only one cycle, COOPO degenerates to regular offline-to-online RL. 
The main contributions can be summarized as: 1) We propose a generalized cyclic framework alternating KL-regularized offline training and adaptive online fine-tuning with periodic dataset anchoring to prevent distributional shift and catastrophic forgetting. 2) We theoretically show the performance difference bound, presenting per-cycle return increase under data coverage. Additionally, COOPO achieves better online sample efficiency, surpassing pure online RL algorithms. Please see Table~\ref{table:complexity_com} for method comparison. 3) We show that COOPO achieves competitive or superior performance compared to state-of-the-art baselines on the widely-used D4RL benchmarks and demonstrate its effectiveness in the offline-to-online RL setting.

\noindent\textbf{Control-Theoretic and CPS Perspective.}
From a control-theoretic viewpoint, the proposed COOPO framework can be interpreted as a closed-loop adaptive control mechanism operating over hybrid data and physical dynamics. 
In traditional control, system stability is maintained by ensuring that feedback updates stay within a bounded region of operation. COOPO achieves a similar effect through its cyclic KL-regularized updates, which restrict policy divergence between consecutive cycles, thus enforcing a form of discrete-time stability in policy space. 
Even under imperfect sensing or cyber perturbations, COOPO's repeated realignment to the offline dataset functions as a corrective stabilizing term that regularly updates or corrects its estimate of the system’s internal state in a robust and fault-tolerant control system. 


\begin{figure}
    \centering
    \includegraphics[width=0.8\linewidth]{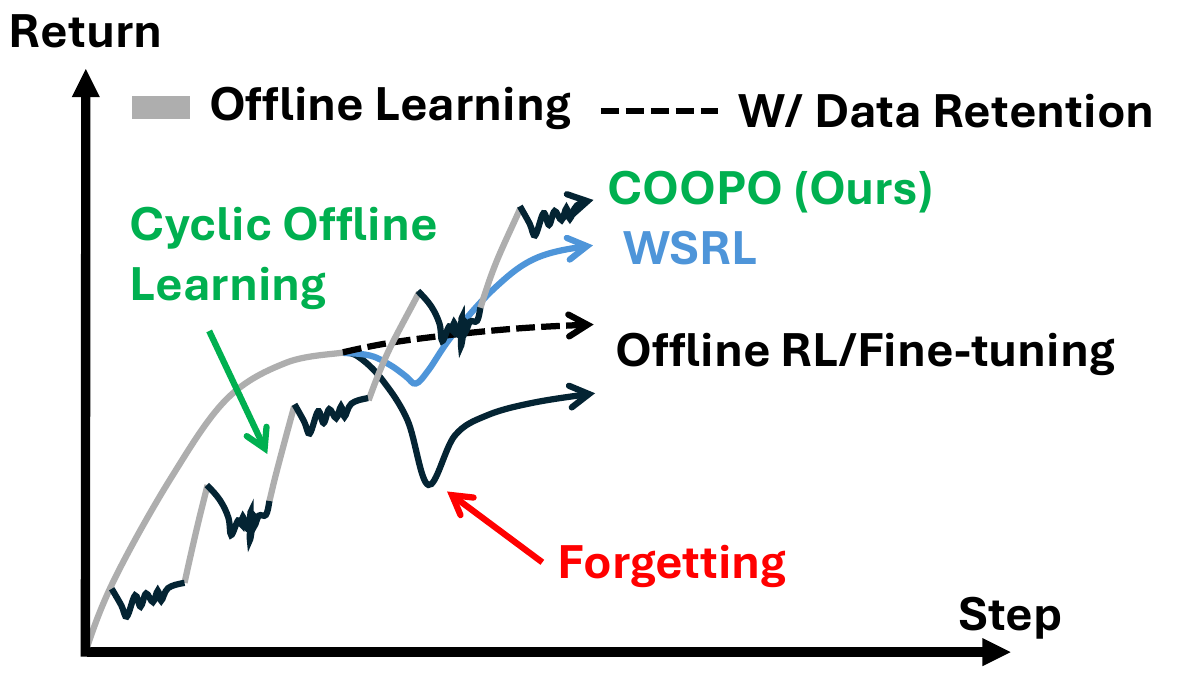}
    \caption{Cyclic offline-online policy optimization: Traditional methods pre-train offline and then fine-tune online, but a single offline phase often leads to \textcolor{red}{catastrophic forgetting} and distributional shift. While \textcolor{blue}{warm-up} phases can partially mitigate this, they fall short in preserving long-term stability. COOPO addresses this by leveraging \textcolor{green}{cyclic synergy} between offline re-training and online fine-tuning, repeatedly anchoring the policy to the offline dataset to enhance robustness and prevent forgetting.}
    \label{fig:coopo_curve}
    \vspace{-0.2in}
\end{figure}


\begin{table}[h]
\caption{Complexity Comparison}
\begin{center}
\begin{threeparttable}
\begin{tabular}{c c}
    \toprule
    \textbf{Method} & \textbf{Sample Complexity} \\ \midrule
      Online PPO   &  $\tilde{\mathcal{O}}(\frac{H^4\varsigma}{\varepsilon^2})$                     \\
      Offline AWAC   &  $\tilde{\mathcal{O}}(\frac{c_1^2}{\varepsilon^2})$                                   \\ 
      1-cycle Hybrid & $\tilde{\mathcal{O}}(\frac{c_1^2}{\varepsilon^2}+\frac{H^4\varsigma}{\varepsilon^2})$ \\
      \textbf{COOPO} & $\tilde{\mathcal{O}}(\frac{c_1^2}{\varepsilon^2}\text{log}\frac{1}{\varepsilon}+\frac{H^3\varsigma}{\varepsilon^2}\text{log}\frac{1}{\varepsilon})$\\
      \bottomrule
\end{tabular}
\begin{tablenotes}
$H$: the episode length, $\varsigma>0$: policy class param, $\varepsilon>0$: the accuracy.
$c_1>0$: a constant related to offline learning.
\end{tablenotes}
\end{threeparttable}
\end{center}
\label{table:complexity_com}
\vspace{-0.1in}
\end{table}

\section{Key Related Work}

\textbf{Online RL.} Online RL typically includes two categories, i.e., on-policy and off-policy algorithms. The key difference lies in how these methods update their policies. On-policy ones~\cite{schulman2017proximal} update policies using data collected by their current behavior policies, ignoring any data generated by history behavior policies and leading to high sample complexity. In contrast, off-policy methods~\cite{haarnoja2018soft} allow policies to learn from experience produced by prior policies, resulting in high sample efficiency. However, when real-world problems are complex, such as transportation~\cite{wang2021deep} and biological systems~\cite{neftci2019reinforcement}, the requirement of massive online interactions still makes online RL inapplicable to solving them.

\noindent\textbf{Offline RL.} Different from online RL, offline RL is inaccessible to any online environment and learns policies solely from pre-collected data. A central challenge in offline RL is value overestimation, such that pessimistic updating strategies have been developed to address the distribution shift problem~\cite{kumar2020conservative,kostrikov2021offline}. Model-free offline RL adopts behavior-constrained approaches to regularize the learned policy to stay close to the data collecting policy~\cite{wu2019behavior,fujimoto2019off,ghosh2022offline}. A recent work found that the inherent conservatism of some on-policy algorithms can help offline RL methods attenuate the overestimation and develop behavior PPO~\cite{zhuang2023behavior}.
When turning to model-based RL, existing methods apply a parameterized model to estimate states and then update policies in a pessimistic way~\cite{kidambi2020morel,yu2020mopo}. Though distributional shift is to some extent resolved, offline RL methods cannot always ensure decent performance, particularly when the data quality is poor, significantly underperforming online RL approaches. Additionally, evaluating learned policies is also challenging if the online environment is unavailable.

\noindent\textbf{Offline-to-Online RL:} The emerging offline-to-online RL has drawn considerable attention due to the synergy of advantages from offline and online RL. The mixed setting was introduced in~\cite{nair2020awac} to learn the policy offline first from the pre-collected dataset and then online via interactions with the environment. More recently, a calibrated Q-learning algorithm was shown to be effective for online fine-tuning~\cite{nakamoto2023cal}, leading to high sample efficiency. Different from value-based methods, \cite{lei2023uni} unified online and offline deep RL with multi-step on-policy optimization without introducing extra conservatism or regularization. They leveraged diverse ensemble policies to address the distributional shift issue. Other methods, such as adaptive policy learning framework~\cite{zheng2023adaptive}, policy-only RL fine-tuning~\cite{xiao2025efficient}, ensemble-based offline-to-online RL~\cite{zhao2023improving}, and policy re-evaluation and value alignment~\cite{luo2024optimistic}, have been researched to bridge the technical gaps in offline-to-online settings.

Another line of research does not require any offline pretraining, instead directly integrates offline datasets with the online training~\cite{song2022hybrid,ball2023efficient,huang2025augmenting}, which gives rise to questions about the efficacy of RL fine-tuning. However, a more recent work~\cite{zhou2024efficient} has proposed a simple yet efficient fine-tuning method without retention of offline data by having a warmup phase at the beginning. 
To further address exploration challenges in long-horizon and sparse-reward tasks, Q-chunking~\cite{li2025reinforcement} adopts action chunking by directly running RL in a ``chunked" action space instead of a single action to maximize the sample efficiency. However, the compounding error, along with a long horizon, may negatively affect the performance. Also, determining the action length depends highly on empirical tasks, requiring manual tuning efforts.

\section{Preliminaries}
\noindent\textbf{Markov Decision Process.}
This work considers a finite-horizon Markov Decision Process (MDP) with discounted reward defined by the tuple $\mathcal{M}=(\mathcal{S}, \mathcal{A}, \mathcal{P}, r, d_0, \gamma)$, where $\mathcal{S}$ and $\mathcal{A}$ represent the sets of states and actions respectively, $\mathcal{P}:\mathcal{S}\times\mathcal{A}\times\mathcal{S}\to [0,1]$ is the transition probability function, $r:\mathcal{S}\times\mathcal{A}\to\mathbb{R}$ is the reward function, $d_0$ is the initial state distribution of environment, and $\gamma\in[0,1]$ is the discount factor. A stochastic policy expressed by $\pi (a|s):\mathcal{S}\to\mathcal{A}$, defines a mapping from the state $s$ to a probability distribution over actions $a$. Denote by $h$ a specific time step and by $H$ the horizon length in a trajectory.
We then also define the stationary state distribution under the policy $\pi$ as $d^\pi(s)=(1-\gamma)\sum_{h=0}^H\gamma^hp(s_h=s)$, where $p$ signifies the probability.
Reinforcement learning aims to choose a policy that is able to maximize the expected discounted cumulative rewards $J(\pi)=\mathbb{E}_{\tau\sim\pi}[\sum_{h=0}^H\gamma^hr(s_h,a_h)]$, where $\tau\sim\pi$ signifies a trajectory sampled according to $s_0\sim d_0$, $a_h\sim\pi(\cdot|s_h)$, and $s_{h+1}\sim p(\cdot|s_h, a_h)$. 
Additionally, we define the state value function of the policy $\pi$ as $V^\pi(s)=\mathbb{E}_{\tau\sim\pi}[\sum_{h=0}^H\gamma^hr(s_h,a_h)|s_0=s]$, the state-action value function, i.e., $Q$-function, as $Q^\pi(s,a)=\mathbb{E}_{\tau\sim\pi}[\sum_{h=0}^H\gamma^tr(s_h,a_h)|s_0=s, a_0=a]$, and the critical advantage function as $A^\pi(s,a)=Q^\pi(s,a)-V^\pi(s)$. This intuitively assesses how much better it is by taking action $a$ than the average.

\noindent\textbf{Online RL.} We focus on on-policy policy optimization algorithms, which estimate policy gradients and apply first-order stochastic gradient ascent. PPO, one of the most popular choices, is favored for its strong performance, simplicity, and theoretical grounding via a policy improvement lower bound. It constrains policy updates using a clipping heuristic, leading to widely used PPO-Clip variant~\cite{huang2024ppo}. At each update, the following objective is optimized:
\begin{equation}\label{eq_1}
\begin{split}
    \mathcal{L}^{PPO}_{\text{actor}}(\theta)&=\mathbb{E}_{\pi_{\text{old}}}[\textnormal{min}(c_h(\theta)\hat{A}^{\pi_h}(s,a),\\&\text{clip}(c_h(\theta),1-\iota,1+\iota)\hat{A}^{\pi_h}(s,a))],
\end{split}
\end{equation}
where $c_h(\theta)=\frac{\pi_\theta(a_h|s_h)} {\pi_{\textnormal{old}}(a_h|s_h)}, \textnormal{clip}(o,q,v)=\textnormal{min}(\textnormal{max}(o,q),v)$. $\hat{A}^{\pi_h}(s,a)$ is an estimator of the advantage function at time step $h$. The clipping function in PPO ensures the probability ratio between current and new policies stays within $[1-\iota, 1+\iota]$, enforcing stable updates. guarantees a lower bound on the surrogate loss. In practice, using a small learning rate and many time steps helps PPO approximate this objective reliably. While PPO is used in COOPO, other trust-region methods (e.g., TRPO~\cite{schulman2015trust}) or off-policy algorithms like SAC and DDPG~\cite{lillicrap2015continuous} can also be used for online fine-tuning.

\noindent\textbf{Offline RL.} In offline RL, the agent can only access a pre-collected dataset with transitions $\mathcal{D}=\{(s_t,a_t,s_{j+1},r_j)_{j=1}^N\}$, generated by an unknown behavior policy $\pi_\beta$. Unlike online RL, offline RL aims to learn an optimal policy directly from a static dataset $\mathcal{D}$. A key challenge is the distributional shift between the learned policy $\pi$ and the behavior policy $\pi_\beta$. When $\pi$ chooses actions not well represented in $\mathcal{D}$, it can lead to inaccurate value estimates. To address this, a behavior constraint is commonly applied to keep $\pi$ close to $\pi_\beta$, typically in the following form:
\begin{equation}\label{eq_2}
    \text{max}_\pi \mathbb{E}_{s\sim\mathcal{D}}[\mathbb{E}_{\pi(a|s)}[Q(s,a)]], \; \text{s.t.} D_{KL}(\pi||\pi_\beta)\leq \delta,
\end{equation}
where $D_{KL}(\cdot, \cdot)$ is the Kullback-Leibler divergence, and $\delta>0$ is a threshold. To solve the above optimization problem, we adopt the Advantage Weighted Actor-Critic (AWAC) algorithm~\cite{nair2020awac}, which trains an off-policy critic and an actor under an implicit policy constraint. Following the actor-critic framework, it alternates between evaluating the policy via $Q^\pi$ and improving it by maximizing a weighted objective based on the estimated advantages. With abuse of notation, we denote by $Q^{\pi_t}(s,a)$ the estimated Q-function at time step $t$ (this is episode herein).
Since optimizing $Q^{\pi_t}(s,a)$ is equivalent to optimizing $A^{\pi_t}(s,a)$~\cite{nair2020awac}, Eq.~\ref{eq_3} can be rewritten as
\begin{equation}\label{eq_3}
\begin{split}
\pi_{t+1}&=\text{argmax}_\pi\mathbb{E}_{a\sim\pi(\cdot|s)}[A^{\pi_t}(s,a)]\\& \text{s.t.} D_{KL}(\pi(\cdot|s)||\pi_\beta(\cdot|s))\leq\delta
\end{split}
\end{equation}

To this end, a non-parametric analytical solution is obtained for the actor $\pi$, subsequently followed by projecting it into the parametric policy class. We first introduce the Lagrangian by augmenting the loss in Eq.~\ref{eq_3} with a multiplier $\lambda>0$ such that
\begin{equation}\label{eq_4}
    \mathcal{L}(\pi, \lambda) = \mathbb{E}_{a\sim\pi(\cdot|s)}[A^{\pi_t}(s,a)]+\lambda(\delta-D_{KL}(\pi(\cdot|s)||\pi_\beta(\cdot|s))).
\end{equation}
Pertaining to closed-form solutions in~\cite{peng2019advantage,nair2020awac}, we can attain $
    \pi^*(a|s)=\frac{1}{Z(s)}\pi_\beta(a|s)\text{exp}(\frac{1}{\lambda}A^{\pi_t}(s,a))$,
where $Z(s)=\int_a\pi(a|s)\text{exp}(\frac{1}{\lambda}A^{\pi_t}(s,a))$ is a partition function. We then parameterize the policy with a deep neural network $\theta\in\mathbb{R}^n$, and project the non-parametric solution into policy space. Mathematically, this is via minimizing the KL divergence between $\pi_\theta$ and $\pi^*$ such that
\[\text{argmin}_\theta\mathbb{E}_{s\sim\mathcal{D}}[D_{KL}(\pi^*(\cdot|s)||\pi_\theta(\cdot|s))].
\]
With some mathematical manipulations as in~\cite{peng2019advantage}, the parameter update at time step $t+1$ is expressed:
\begin{equation}\label{eq_7}
    \theta_{t+1}=\text{argmax}_{\theta}\mathbb{E}_{(s,a)\sim\mathcal{D}}[\text{log}\pi_\theta(a|s)\text{exp}(\frac{1}{\lambda}A^{\pi_t}(s,a))].
\end{equation}
Updating $\theta$ amounts to the maximum likelihood estimation where the labels are accessible by weighting the state-action pairs observed in the dataset by the predicted advantages. This can simply be done by sampling $(s,a)$ from $\mathcal{D}$. Though in our work, offline training is done through AWAC, other offline RL algorithms can alternatively be used to replace AWAC, possibly with some minor algorithmic adjustments. 

A key requirement in cyber-physical systems is that the closed-loop controller ~\cite{khalil2002nonlinear} evolves in a stable and predictable manner, avoiding large or abrupt changes that may violate safety constraints. 
In classical adaptive control ~\cite{astrom1995adaptive}, this requirement is typically enforced by bounding the incremental change of controller parameters or by maintaining a sufficient stability margin. 
In COOPO, an analogous stability mechanism naturally emerges from the KL-regularization imposed in the offline policy update.

Due to each offline iteration, COOPO restricts the divergence between the updated policy $\pi_{k+1}$ and the reference policy $\pi_k$ through a KL trust region:
\begin{equation}
\begin{split}
D_{KL}(\pi_{k+1}\,\Vert\,\pi_k) \le \delta.
\end{split}
\end{equation}

This constraint ensures that the new controller remains close to the previous one. Moreover, by Pinsker’s inequality~\cite{cover1999elements}, this directly bounds on the total-variation distance between the two policies:
\begin{equation}
\begin{split}
\|\pi_{k+1} - \pi_k\|_{\mathrm{TV}}
    \;\le\;
    \sqrt{\tfrac{1}{2} D_{KL}(\pi_{k+1}\Vert \pi_k)}
    \;\le\;
    \sqrt{\tfrac{\delta}{2}}.
\end{split}
\end{equation}
Thus, policy changes across offline–online cycles remain uniformly bounded, resulting in a smooth and safe evolution of the closed-loop behavior.


\section{Proposed Method}
\begin{figure}
    \centering
    \includegraphics[width=\linewidth]{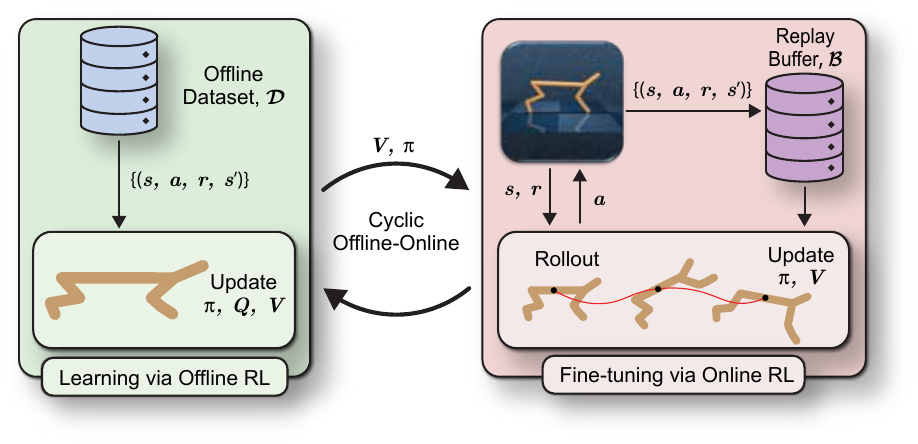}
    \caption{Schematic Diagram of COOPO: COOPO cyclically alternates KL-regularized offline training and trust-region online fine-tuning to eliminate distribution drift and catastrophic forgetting. When there is only one cycle, COOPO degenerates to vanilla offline-to-online RL.}
    \label{fig:coopo}
    \vspace{-0.1in}
\end{figure}
We start with the algorithmic framework in the sequel. Relevant notations will be defined alongside the analysis and the proofs are deferred to Appendix.

\noindent\textbf{COOPO.} Algorithm~\ref{alg:coopo} outlines the core steps of COOPO. While COOPO can theoretically adopt any RL algorithms, actor-critic methods are preferred in practice for smoother offline-to-online transitions. The actor and critic, initially trained offline, are fine-tuned online for performance gains. Unlike conventional approaches that stop after online fine-tuning risking distributional shift and catastrophic forgetting. COOPO cyclically returns to the offline dataset, anchoring the policy to supported states and preserving prior knowledge. As shown in Figure~\ref{fig:coopo}, each cycle includes:
\textbf{Offline Phase}: Optimize policies using KL-regularized advantage-weighted updates, minimizing distributional shift via a loss combining importance-weighted behavior cloning and KL penalties;
\textbf{Online Phase}: Fine-tune policies via any exploratory policy gradient method (e.g., PPO, TRPO, SAC), ensuring stable exploration with clipping mechanisms or divergence constraints.

\noindent\textbf{KL-regularized AWAC with value function.} We parameterize the Q-function and value-function with $\phi$ and $\psi$, initializing them from the models used in online fine-tuning. Lines 4–11 in Algorithm~\ref{alg:coopo} outline the offline updates, which largely follow AWAC, with a key difference: we introduce an additional value model (Line 8) to support a seamless transition to online learning. This is essential since our online phase uses PPO, which relies on a value function. This also enables us to calculate advantage $\hat{A}^{\pi_{\theta_e}^{\text{off}}}(s,a))$ in a slightly different way (Line 9) from that used in~\cite{nair2020awac}. Particularly, we also
incorporate an explicit KL divergence penalty $\mathbb{E}_{s\sim\mathcal{D}}[D_{KL}(\pi^{\text{off}}_\theta||\pi_{\theta_e}^{\text{off}})(\cdot|s)]$ to constrain the current policy $\pi_{\theta_e}^{\text{off}}$ and the next learned policy $\pi^{\text{off}}_\theta$ (Line 10), directly controlled by a coefficient $\lambda$. While Eq.~\ref{eq_7}, implicitly constrains the learned policy $\pi$ to stay close to the behavior policy $\pi_\beta$, KL-regularized AWAC explicitly enforces a trust region, ensuring updates remain near previous policies and dataset-supported actions. Unlike the implicit constraint in Eq.~\ref{eq_7} where $\lambda$ only scales advantages, explicit KL control reduces sensitivity to hyperparameters and better prevents distributional shift and catastrophic forgetting. Moreover, incorporating a value function $V^{\text{off}}_\psi$ helps avoid out-of-distribution (OOD) action estimates, similar to IQL, enabling more robust learning from static datasets. 


\noindent\textbf{Transition to online RL.} Line 12 shows the critical transition from offline to online settings, involving the value and policy models. Concretely, the last iterates of KL-regularized AWAC of $\psi$ and $\theta$ are the initialization of online RL, for which we resort to the widespread PPO algorithm (Lines 14-20). In Line 21, the cyclic models at the $k$-th cycle are updated by using fine-tuned value and policy models. As the parametric Q model is not fine-tuned in the online phase, the cyclic Q model is obtained by directly adopting the last iterate of Q model in the offline phase. 

\noindent\textbf{Non-trivial synergy in COOPO.} COOPO cyclically integrates KL-regularized AWAC and trust-region PPO to address limitations of naive hybrids. To maintain critic stability across cycles, COOPO freezes the offline-trained Q-function during the online PPO phase. We empirically observed that jointly updating the critic online increases variance and value drift due to distribution mismatch between offline and online states. Freezing the critic ensures a consistent advantage baseline and aligns with findings in prior offline-to-online RL literature, where maintaining a stable value estimator is essential for preventing policy collapse during online fine-tuning.

The KL constraint anchors the policy to the offline dataset, reducing distributional drift, while PPO’s clipped updates enable stable online fine-tuning. The KL constraint plays a dual role: (i) it bounds inter-cycle policy deviation, thereby ensuring stable improvement at each step, and (ii) it mitigates overfitting to the offline dataset by preventing aggressive policy updates that diverge from known-safe behaviors. This mechanism is particularly important in CPS settings, where abrupt controller changes can violate safety margins. Repeated offline phases progressively average advantage errors and stabilize training. This closed-loop alternation forms a feedback mechanism: \textit{each cycle resets policy divergence with KL penalties and refines policy using online signals}. This synergy lets AWAC stabilize PPO’s exploration and PPO guide AWAC updates. COOPO also amortizes offline data cost across cycles, thus improving efficiency. Theoretically, it guarantees performance bounds and logarithmic cycle complexity; empirically, it achieves strong sample efficiency on benchmarks, demonstrating its effective design.

To isolate the effect of the static offline dataset, online rollouts are intentionally excluded during the offline update. This design ensures that the offline update has enough contribution into the current policy by using the better data : when the online policy becomes locally trapped or drifts due to limited exploration, the offline step pulls the agent back toward high-quality behaviors supported by the dataset. By avoiding the accumulation of mixed with the online samples, COOPO prevents dataset degradation and maintains a stable reference distribution, which give a clearer attribution of performance gains to the cyclic update mechanism. 
Incorporating replayed online samples remains feasible and may further accelerate adaptation, which we leave for future work.

\noindent\textbf{Theoretical Analysis.}
To start with the analysis, we denote by $\pi_k:=\pi_{\theta_k}$ the policy at the start of cycle $k$. To ease the notation, we also denote by $\pi_{k+1/2}$ and $\pi_{k+1}$ the policies after the offline phase and online phase of cycle $k$, respectively. The state distribution in the offline dataset $\mathcal{D}$ is defined as $\rho_\beta(s)$. In both offline and online phases, the advantage function plays a central role in evaluating predicted actions, With abuse of notations, we still adopt $A^{\pi_k}(s,a)$ (or $\hat{A}^{\pi_k}(s,a)$) as the advantage (or estimated advantage) throughout the rest of paper. In a rigorous sense, the notations of advantage functions differ slightly as in Algorithm~\ref{alg:coopo}, but we unify them for the ease of analysis. Analogously, we also use the same value to upper bound advantages in these two phases, i.e., $\epsilon_k=\text{max}_{s,a}|A^{\pi}_k(s,a)|$. Throughout the analysis, $\mathcal{O}(\cdot)$ represents the standard big O notation for complexity. $\tilde{\mathcal{O}}(\cdot)$ notation simplifies $\mathcal{O}(\cdot)$ notation by ignoring logarithmic factors. 
In what follows, we present a key assumption on the distribution shift coefficient $C$, which helps characterize the bound.
\begin{assumption}\label{assumption_1}
    There exists a constant $1\leq C<\infty$ such that $C=\text{max}_s\frac{d^{\pi_k}(s)}{\rho_\beta(s)}$ if for all state $s$ with $d^{\pi_k}(s)>0$, $\rho_\beta(s)>0$.
\end{assumption}
Assumption~\ref{assumption_1} bounds the distance between the visitation distributions of the behavior policy $\pi_\beta$ and any learned policy $\pi_k$, capturing the data coverage to ensure their closeness, similar to single policy concentrability in prior work
~\cite{xie2021policy,rashidinejad2021bridging}.
In Algorithm~\ref{alg:coopo}, the offline learning algorithm is KL-regularized AWAC, instead of AWAC, which leads to Eq.~\ref{eq_4}. Hence, we will show under the KL-regularized AWAC, such a relationship still holds. We present an auxiliary technical lemma to reveal this.
\begin{lemma}\label{lemma_1}
    With KL-AWAC objective function
    \begin{equation}
    \begin{split}
        \pi_{new}&=\text{argmax}_\pi\mathbb{E}_{s,a\sim\mathcal{D}}[\text{log}\pi(a|s)\cdot w(s,a)]\\&-\lambda\mathbb{E}_{s\sim\mathcal{D}}[D_{KL}(\pi(\cdot|s)||\pi_{old}(\cdot|s)],
    \end{split}
    \end{equation}
    where $w(s,a)=\text{exp}\bigg(\frac{A^{\pi^{old}}(s,a)}{\lambda}\bigg)$ is the advantage weight, and $\lambda>0$ is a hyperparameter to control the regularization length, we have the following relationship
    \begin{equation}
        \mathbb{E}_{a\sim\pi^*(\cdot|s)}[A_{\pi_{old}}(s,a)]=\lambda D_{KL}(\pi^*||\pi_{old})(s)+\lambda \text{log}Z(s),
    \end{equation}
    where $Z(s)=\mathbb{E}_{a\sim\pi_{old}(\cdot|s)}\bigg[\text{exp}\bigg(\frac{A^{\pi_{old}}(s,a)}{\lambda}\bigg)\bigg]$.
\end{lemma}
\begin{proof}
    The KL divergence term can expand to 
    \[D_{KL}(\pi(\cdot|s)||\pi_{old}(\cdot|s)=\mathbb{E}_{a\sim\pi(\cdot|s)}[\text{log}\pi(a|s)-\text{log}\pi_{old}(a|s)].\] Substituting it into the objective yields \begin{equation}
        \begin{split}\eta(\pi)&=\mathbb{E}_{s,a\sim\mathcal{D}}[\text{log}\pi(a|s)\cdot w(s,a)]\\&-\lambda\mathbb{E}_{a\sim\pi(\cdot|s)}[\text{log}\pi(a|s)-\text{log}\pi_{old}(a|s)].
        \end{split}
    \end{equation}

\begin{algorithm}[t]
\caption{COOPO}
\label{alg:coopo}
\small
\begin{algorithmic}[1]

\State \textbf{Input:} Offline dataset $\mathcal{D}=\{(s,a,s',r)_j\}$, cycle $K$, offline epochs $E$, online episodes $T$, KL coefficient $\lambda$, horizon $H$
\State \textbf{Initialize:} $\pi_{\theta_1}, Q_{\phi_1}, V_{\psi_1}$
\State \textbf{Output:} $\pi_{\theta_K}$

\For{$k=1,\ldots,K$}

\State $Q^{\mathrm{off}}_{\phi_1} \gets Q_{\phi_k}$,
$\pi^{\mathrm{off}}_{\theta_1} \gets \pi_{\theta_k}$,
$V^{\mathrm{off}}_{\psi_1} \gets V_{\psi_k}$

\Statex \textit{Offline learning}

\For{$e=1,\ldots,E$}

\State Sample batch $(s,a,s',r)\sim \mathcal{D}$

\State $y \gets r(s,a) + \gamma \mathbb{E}_{s',a'}
\left[Q^{\mathrm{off}}_{\phi_e}(s',a')\right]$

\State $\phi_{e+1} \gets
\arg\min_{\phi}
\mathbb{E}_{\mathcal{D}}
\left[
(Q^{\mathrm{off}}_{\phi}(s,a)-y)^2
\right]$

\State $\psi_{e+1} \gets
\arg\min_{\psi}
\mathbb{E}_{\mathcal{D}}
\left[
(V^{\mathrm{off}}_{\psi}(s)-y)^2
\right]$

\State $\hat A^{\pi^{\mathrm{off}}_{\theta_e}}(s,a)
\gets
Q^{\mathrm{off}}_{\phi}(s,a)
-
V^{\mathrm{off}}_{\psi}(s)$

\State $\theta_{e+1} \gets
\begin{aligned}[t]
\arg\max_{\theta} \;
&\mathbb{E}_{s,a\sim\mathcal{D}}
\left[
\log \pi^{\mathrm{off}}_{\theta}(a|s)
\exp\!\left(
\frac{1}{\lambda}
\hat A^{\pi^{\mathrm{off}}_{\theta_e}}(s,a)
\right)
\right] \\
&-
\lambda
\mathbb{E}_{s\sim\mathcal{D}}
\left[
D_{KL}
\big(
\pi^{\mathrm{off}}_{\theta}(\cdot|s)
\|
\pi^{\mathrm{off}}_{\theta_e}(\cdot|s)
\big)
\right]
\end{aligned}$

\EndFor

\State $V^{\mathrm{on}}_{\psi_1} \gets V^{\mathrm{off}}_{\psi_E}$,
$\pi^{\mathrm{on}}_{\theta_1} \gets \pi^{\mathrm{off}}_{\theta_E}$

\Statex \textit{Online fine-tuning}

\For{$t=1,\ldots,T$}

\State Collect trajectories $\mathcal{R}_t=\{\tau_i\}$ using $\pi^{\mathrm{on}}_{\theta_t}$

\State Compute reward-to-go $\hat R_h$ within horizon $H$

\State Compute advantages
$\hat A^{\pi^{\mathrm{on}}_{\theta_t}}_h$
using $V^{\mathrm{on}}_{\psi_t}$

\State Update policy parameter $\theta_{t+1}$ via Eq.~\eqref{eq_1}

\State $\psi_{t+1} \gets
\arg\min_{\psi}
\mathbb{E}_{\mathcal{R}_t}
\left[
\frac{1}{H}
\sum_{h=0}^{H}
(V^{\mathrm{on}}_{\psi}(s_h)-\hat R_h)^2
\right]$

\EndFor

\State $Q_{\phi_{k+1}} \gets Q^{\mathrm{off}}_{\phi_E}$

\State $V_{\psi_{k+1}} \gets V^{\mathrm{on}}_{\psi_T}$

\State $\pi_{\theta_{k+1}} \gets \pi^{\mathrm{on}}_{\theta_T}$

\EndFor

\end{algorithmic}
\end{algorithm}

We now solve the optimal policy $\pi^*$ by taking the functional derivative of $\eta(\pi)$ with respect to $\pi(\cdot|s)$ and set to zero:
\begin{equation}
\begin{split}
    \frac{d\eta}{d\pi}=\frac{w(s,a)}{\pi(s|a)}-\lambda(\text{log}\pi(a|s)-\text{log}\pi_{old}(a|s)+1)+\nu(s)=0,
\end{split}
\end{equation}
where $\nu(s)$ us a Lagrangian multiplier enforcing $\int\pi(a|s)da=1$. Solving for $\pi(a|s)$, we have
\begin{equation}
    \pi^*(a|s)\propto \pi_{old}(a|s)\text{exp}(\frac{A^{\pi_{old}}(s,a)}{\lambda})\text{exp}(-\frac{\nu(s)+\lambda}{\lambda}).
\end{equation}
Define the normalization constant as $Z(s)$: we then have
\begin{equation}
    \pi^*(a|s)=\frac{1}{Z(s)}\pi_{old}(a|s)\text{exp}\bigg(\frac{A^{\pi_{old}}(s,a)}{\lambda}\bigg),
\end{equation}
where $Z(s)=\mathbb{E}_{a\sim\pi_{old}(\cdot|s)}\bigg[\text{exp}\bigg(\frac{A^{\pi_{old}}(s,a)}{\lambda}\bigg)\bigg]$
This resembles the same formula in the analysis in~\cite{peng2019advantage}. 
We then take the $\text{log}$ on both sides of $\pi^*(a|s)$, obtaining
\begin{equation}
    \text{log}\pi^*(a|s)=\text{log}\pi_{old}(a|s)+\frac{A^{\pi_{old}}(s,a)}{\lambda}-\text{log}Z(s).
\end{equation}
Now we compute the KL divergence:
\begin{equation}
\begin{split}
    D_{KL}(\pi^*||\pi_{old})(s)&=\mathbb{E}_{a\sim\pi^*(\cdot|s)}[\text{log}\pi^*(a|s)-\text{log}\pi_{old}(a|s)]\\&=\mathbb{E}_{a\sim\pi^*(\cdot|s)}\bigg[\frac{A^{\pi_{old}}(s,a)}{\lambda}-\text{log}Z(s)\bigg].
\end{split}
\end{equation}
Rearranging to isolate the expected advantage completes the proof.
\end{proof}
We next state the lower bound for the difference between $\pi_{k}$ and $\pi_{k+1}$ in COOPO.
\begin{theorem}\label{theorem_1} (Performance Difference Bound)
    Let Assumption~\ref{assumption_1} hold. The advantage error in the offline phase is defined such that $\epsilon^\text{adv}_k=\text{max}_{s,a}|\hat{A}^{\pi_k}(s,a)-A^{\pi_k}(s,a)|$ for any policy $\pi_k$. Hence, for any cycle $k$, two consecutive policies $\pi_k, \pi_{k+1}$ produced by COOPO, satisfy the following:
    \begin{equation}\label{eq_8}
        J(\pi_{k+1})-J(\pi_k)\geq \frac{\lambda}{1-\gamma}G^{\text{off}}_k-\Lambda_k,
    \end{equation}
    where \[G^{\text{off}}_k=\mathbb{E}_{s\sim d^{\pi_k}}[D_{KL}(\pi_{k+1/2}||\pi_k)(s)],\] \[\Lambda_k=\frac{\epsilon^\text{adv}_k}{1-\gamma}+\mathbb{E}_{s\sim d^{\pi_k}}[\mathbb{E}_{a\sim\pi_{k+1/2}}[\frac{2\gamma\epsilon_k\zeta_k}{(1-\gamma)^2}]]+\frac{4\epsilon_k\gamma\alpha^2_k}{(1-\gamma)^2},\] $\zeta_k=\text{max}_sD_{TV}(\pi_{k+1/2}||\pi_k)(s)$, $\alpha_k=\text{max}_sD_{TV}(\pi_{k+1}||\pi_{k+1/2})(s)$. $D_{TV}(\cdot||\cdot)$ is the total variation distance. 
\end{theorem}
\begin{proof}
    We divide the proof into two parts: offline phase (KL-reguarlized AWAC) and offline phase (PPO). Recalling Lemma~\ref{lemma_1}, we obtain the following relationship:
    \begin{equation}
    \begin{split}
        &\mathbb{E}_{s\sim d^{\pi_k}}[\mathbb{E}_{a\sim\pi_{k+1/2}}[\hat{A}^{\pi_k}(s,a)]]\\&=\mathbb{E}_{s\sim d^{\pi_k}}[\lambda\cdot D_{KL}(\pi_{k+1/2}||\pi_k)(s)+\lambda \text{log}Z(s)],
    \end{split}
    \end{equation}
    where $Z(s)=\mathbb{E}_{a\sim\pi_k}[\text{exp}(\frac{1}{\lambda}\hat{A}^{\pi_k}(s,a))]$. Now we relate this to the true advantage such that 
    \begin{equation}
        \begin{split}
            &\mathbb{E}_{s\sim d^{\pi_k}}[\mathbb{E}_{a\sim\pi_{k+1/2}}[A^{\pi_k}(s,a)]] =  \mathbb{E}_{s\sim d^{\pi_k}}[\mathbb{E}_{a\sim\pi_{k+1/2}}[\hat{A}^{\pi_k}(s,a)]] \\&-\mathbb{E}_{s\sim d^{\pi_k}}[\mathbb{E}_{a\sim\pi_{k+1/2}}[\hat{A}^{\pi_k}(s,a)-A^{\pi_k}(s,a)]].
        \end{split}
    \end{equation}
    The second term of the last equality can be bounded by $\epsilon_k^{adv}$ such that
    \begin{equation}\label{eq_11}
    \begin{split}
        &\mathbb{E}_{s\sim d^{\pi_k}}[\mathbb{E}_{a\sim\pi_{k+1/2}}[A^{\pi_k}(s,a)]]\\&\geq \mathbb{E}_{s\sim d^{\pi_k}}[\lambda\cdot D_{KL}(\pi_{k+1/2}||\pi_k)(s)+\lambda \text{log}Z(s)]-\epsilon_k^{adv}.
    \end{split}
    \end{equation}
    Note that $\text{log}Z(s)\geq 0$ because $\text{exp}(\frac{1}{\lambda}\hat{A}^{\pi_k}(s,a))\geq 0$ and by Jensen's inequality. To construct the performance improvement, we leverage a well-known Performance Difference Lemma~\cite{kakade2002approximately} as follows:
    \begin{equation}
        J(\pi_{k+1/2})-J(\pi_k)=\frac{1}{1-\gamma}\mathbb{E}_{s\sim d^{\pi_{k+1/2}}}[\mathbb{E}_{a\sim \pi_{k+1/2}}[A^{\pi_k}(s,a)]].
    \end{equation}
    But this expectation is under $d^{\pi_{k+1/2}}$ instead of $d^{\pi_{k}}$. We then connect these distributions using the following inequality from Corollary 1 in~\cite{achiam2017constrained}
    \begin{equation}
        \begin{split}
            &J(\pi')-J(\pi)\geq \mathbb{E}_{s\sim d^{\pi}}[\mathbb{E}_{a\sim\pi'(\cdot|s)}[\frac{1}{1-\gamma}A^\pi(s,a)\\&-\frac{2\gamma\epsilon^{\pi'}}{(1-\gamma)^2}\text{max}_sD_{TV}(\pi'||\pi)(s)]].
        \end{split}
    \end{equation}
    Setting $\pi'=\pi_{k+1/2}$ and $\pi=\pi_k$ yields the following relationship:
    \begin{equation}\label{eq_14}
        \begin{split}
            &J(\pi_{k+1/2})-J(\pi_k)\geq \mathbb{E}_{s\sim d^{\pi_k}}[\mathbb{E}_{a\sim\pi_{k+1/2}}[\frac{1}{1-\gamma}A^{\pi_k}(s,a)\\&-\frac{2\gamma\epsilon_k\zeta_k}{(1-\gamma)^2}]].
        \end{split}
    \end{equation}
    Substituting Eq.~\ref{eq_11} into Eq.~\ref{eq_14}, we have
    \begin{equation}\label{eq_15}
        \begin{split}
            &J(\pi_{k+1/2})-J(\pi_k)\geq \frac{1}{1-\gamma}(\lambda\mathbb{E}_{s\sim d^{\pi_k}}[D_{KL}(\pi_{k+1/2}||\pi_k)(s)]\\&-\epsilon_k^{adv}) -\mathbb{E}_{s\sim d^{\pi_k}}[\mathbb{E}_{a\sim\pi_{k+1/2}}[\frac{2\gamma\epsilon_k\zeta_k}{(1-\gamma)^2}]],
        \end{split}
    \end{equation}
    which is the bound after the offline phase. We next proceed for the online phase, where we use PPO to update $\pi_{k+1/2}$ to $\pi_{k+1}$. The PPO update is designed to ensure a trust region. Following Theorem 1 in~\cite{schulman2015trust} and Proposition 1 in~\cite{achiam2017constrained}, we obtain the following relationship
    \begin{equation}\label{eq_16}
    \begin{split}
        &J(\pi_{k+1})-J(\pi_{k+1/2})\geq \\&-\frac{4\epsilon_k\gamma}{(1-\gamma)^2}\text{max}_sD_{TV}(\pi_{k+1}||\pi_{k+1/2})(s).
    \end{split}
    \end{equation}
    Combining Eqs.~\ref{eq_15} and~\ref{eq_16} obtains the desirable result.
\end{proof}
Theorem~\ref{theorem_1} suggests that the error term $\Lambda_k$ is primarily dictated by the advantage estimation error, the total variation distance between $\pi_k$ and $\pi_{k+1/2}$, and the total variation distance between $\pi_{k+1/2}$ and $\pi_{k+1}$. The first one can be connected with KL divergence as it is in the offline phase with the exact constraint in Eq.~\ref{eq_3} by using the Pinsker's inequality~\cite{fedotov2003refinements}. However, the second one cannot be connected to KL divergence, requiring an assumption to upper bound it. In this context, we follow the practice from~\cite{schulman2015trust} such that we use a constant $\alpha>0$.
Eq.~\ref{eq_8} implies the policy performance difference between two consecutive steps is attributed to differences in both offline and online phases. To further quantify the cycle complexity for $K$, we assume that the advantage estimation error $\epsilon^\text{adv}_k$
is bounded by $\bar{\epsilon}^\text{adv}>0$ and that the maximum advantage value $\epsilon_k$ is bounded by $\bar{\epsilon}>0$. We also define the \textit{suboptimality} as $\Delta_k=J(\pi^*)-J(\pi_k)$, where $\pi^*$ is the optimal policy. In Eq.~\ref{eq_8}, we observe the positive gain from offline phase, $G^{\text{off}}_k$, which is assumed to be at least propositional to the current suboptimality: $G^\text{off}_k\geq \kappa \Delta_k$, where $0<\kappa\leq \frac{1-\gamma}{\lambda}$. We briefly justify why this is a reasonable assumption in this work. When the policy is far from optimal (large $\Delta_k$), the offline phase should yield significant improvements (large $G^\text{off}_k$. Conversely, when the policy is near optimal (small $\Delta_k$), improvements are harder to come by and thus $G^\text{off}_k$ is small. Hence, $G^\text{off}_k$ scales with $\Delta_k$. It holds empirically when the dataset $\mathcal{D}$ covers states with positive advantages relative to $\pi_k$. While sparse rewards or poor coverage can violate this, optimistic exploration during online phases may likely restore validity by adding diverse rollouts to $\mathcal{D}$ and guiding it to include high-advantage transitions.
With these in hand, the cycle complexity is stated in the following main result.
\begin{theorem}\label{theorem_2} (Cycle Complexity)
    With definitions in Theorem~\ref{theorem_1}, and $\epsilon^\text{adv}_k\leq\bar{\epsilon}^{\text{adv}}, \epsilon_k\leq \bar{\epsilon}, \alpha_k\leq\alpha, G^\text{off}_k\geq \kappa\Delta_k$, COOPO converges to the neighborhood of $\pi^*$ in a linear convergence rate, i.e., $
        \Delta_K\leq (1-\frac{\lambda\kappa}{1-\gamma})^K\Delta_0+\frac{\Lambda(1-\gamma)}{\lambda\kappa},
$
    where $\Lambda=\frac{\bar{\epsilon}^\text{adv}}{1-\gamma}+\frac{C\gamma\sqrt{2\delta}\bar{\epsilon}}{(1-\gamma)^2}+\frac{4\bar{\epsilon}\gamma\alpha^2}{(1-\gamma)^2}$.
\end{theorem}
\begin{proof}
    Based on Eq.~\ref{eq_8}, we have
    \begin{equation}
        J(\pi^*)-J(\pi_{k+1})\leq J(\pi^*)-J(\pi_{k}) -\frac{\lambda\kappa}{1-\gamma}\Delta_k+\Lambda_k.
    \end{equation}
    The last inequality is due to the assumption $G^\text{off}_k\geq \kappa\Delta_k$. For $\Lambda_k$, based on the conditions, it is immediately obtained that \[\frac{\bar{\epsilon}^\text{adv}_k}{1-\gamma}\leq\frac{\bar{\epsilon}^\text{adv}}{1-\gamma}, \frac{4\bar{\epsilon}_k\alpha^2_k\gamma}{(1-\gamma)^2}\leq\frac{4\bar{\epsilon}\alpha^2\gamma}{(1-\gamma)^2}.\] For the second term in $\Lambda_k$, we would like to bound it with the KL divergence constraint in Eq.~\ref{eq_3}. However, the expectation is over the distribution $d^{\pi_k}(s)$, instead of the dataset distribution $\rho_\mathcal{D}(s)$ as in the offline objective. Based on Assumption~\ref{assumption_1}, we know that, for any function $f(s)\geq 0$
    \begin{equation}
    \begin{split}
        &\mathbb{E}_{s\sim d^{\pi_k}}[f(s)] = \int d^{\pi_k}(s)f(s)ds \\&=  \int \rho_\mathcal{D}(s)\frac{d^{\pi_k}(s)}{\rho_\mathcal{D}(s)}f(s)ds \leq C\int \rho_\mathcal{D}(s)f(s)ds.
    \end{split}
    \end{equation}
    Thus, $\mathbb{E}_{s\sim d^{\pi_k}}[f(s)]\leq C\mathbb{E}_{s\sim \rho_\mathcal{D}}[f(s)]$.
    This results in
    \begin{equation}
    \begin{split}
        &\mathbb{E}_{s\sim d^{\pi_k}}[\mathbb{E}_{a\sim\pi_{k+1/2}}[\frac{2\gamma\epsilon_k\zeta_k}{(1-\gamma)^2}]]\leq \\&\leq C\mathbb{E}_{s\sim \rho_\mathcal{D}}[\mathbb{E}_{a\sim\pi_{k+1/2}}[\frac{2\gamma\epsilon_k\zeta_k}{(1-\gamma)^2}]].
    \end{split}
    \end{equation}
    Pinsker's inequality leads to \[\zeta_k\leq\sqrt{\frac{D_{KL}(\pi_{k+1/2}||\pi_k)(s)}{2}},\] which attains
    \begin{equation}
        \mathbb{E}_{s\sim d^{\pi_k}}[\mathbb{E}_{a\sim\pi_{k+1/2}}[\frac{2\gamma\epsilon_k\zeta_k}{(1-\gamma)^2}]]\leq\frac{\sqrt{2\delta}C\gamma\bar{\epsilon}}{(1-\gamma)^2}.
    \end{equation}
    Hence, the following relationship holds
    \begin{equation}
        \Delta_{k+1}\leq (1-\frac{\lambda\kappa}{1-\gamma})\Delta_k+\Lambda.
    \end{equation}
    Applying the last inequality iteratively from $k=0,...,K$ completes the proof.
\end{proof}
Theorem~\ref{theorem_2} implies the linear convergence to the neighborhood of $\pi^*$, which typically appears in the gradient-based optimization with smooth and strongly convex objectives. However, in this work, we do not have such an assumption for the objective. 
Suppose that the desirable accuracy is $\varepsilon$ for COOPO, i.e., $\Delta_K\leq \varepsilon$. Using the conclusion from Theorem~\ref{theorem_2},
it is obtained that $(1-\frac{\lambda\kappa}{1-\gamma})^K\Delta_0\leq \frac{\varepsilon}{2}$ and $\frac{\Lambda(1-\gamma)}{\lambda\kappa}\leq \frac{\varepsilon}{2}$, 
which is equivalent to $\Lambda\leq\frac{\varepsilon\lambda\kappa}{2(1-\gamma)}$.
Therefore, we have $K\geq \frac{\text{log}(2\Delta_0/\varepsilon)}{\text{log}(1/(1-\frac{\kappa\lambda}{1-\gamma})}\approx \frac{1-\gamma}{\kappa\lambda}\text{log}(\frac{2\Delta_0}{\varepsilon})$. This shows the cycle complexity is \textit{logarithmic} in $\frac{1}{\varepsilon}$, which is favored in general. To ensure the relationship $\Lambda\leq\frac{\varepsilon\lambda\kappa}{2(1-\gamma)}$ and the per-cycle gain $\frac{\kappa\lambda}{1-\gamma}$ to be positive and bounded away from zero, practically speaking, if the offline dataset is good (high coverage leads to $C\to 1$) and the online phase is stable, these two are expected to hold.

So far, we have quantified the cycle complexity $K$ of COOPO’s cyclic framework, but the total sample complexity dependent on both offline and online phases remains to be determined. We now present a result that characterizes COOPO’s key sample complexity. In each cycle, the offline phase involves training for $E$ epochs on a fixed dataset $\mathcal{D}$ of size $|\mathcal{D}|$, resulting in $E|\mathcal{D}|$ offline samples per cycle and a total of $KE|\mathcal{D}|$ offline samples over $K$ cycles. Since the same dataset is reused, the key question is how to set $E$. Offline training must be sufficiently accurate to keep error terms in the performance bound under control, particularly the advantage estimation error $\bar{\epsilon}^\text{adv}$, which which must be small enough to ensure $\Lambda$ to be $\mathcal{O}(\varepsilon)$.
From standard supervised learning convergence, if we use first-order stochastic optimization method (which we have used in COOPO), the error after $E$ epochs satisfies the relationship~\cite{garrigos2023handbook}: $\bar{\epsilon}^\text{adv}=\mathcal{O}(1/\sqrt{E|\mathcal{D}|})$, from which we can obtain the offline samples per cycle.  
In each cycle, the online phase collects $T$ samples such that the total online samples over $K$ cycles is $KT$. In this phase, we run a trust-region for a few episodes to control the error in the improvement. Particularly, we require that the online phase does not degrade the policy significantly and improves upon the offline policy. From the online complexity of policy gradient methods~\cite{qiu2021finite,yuan2022general}, they ensure that the online suboptimality $\varepsilon_\text{on}$ satisfies $\varepsilon_\text{on}=\tilde{\mathcal{O}}(\sqrt{H^3\varsigma/T})$, where $H:=1/(1-\gamma)$ is the effective horizon and $\varsigma>0$ is a policy class parameter. For example, $\varsigma=\tilde{\mathcal{O}}(|\mathcal{S}||\mathcal{A}|)$ if the parametric model for the policy is neural network. Based on this, we can obtain the online sample complexity per cycle. We are now ready to present the main result to characterize the total sample complexity in the following.

\begin{theorem}\label{theorem_3} (Total Sample Complexity)
    Suppose that the offline error satisfies $\bar{\epsilon}^\text{adv}=\mathcal{O}(1/\sqrt{E|\mathcal{D}|})$ and that the online error satisfies $\varepsilon_\text{on}=\tilde{\mathcal{O}}(\sqrt{H^3\varsigma/T})$. 
    With the cycle complexity from Theorem~\ref{theorem_2}, the total sample complexity incurred by COOPO is $\tilde{\mathcal{O}}\bigg(\frac{(1-\gamma)c^2_1}{\lambda\kappa\varepsilon^2}\text{log}\frac{1}{\varepsilon}+\frac{(1-\gamma)H^3\varsigma}{\lambda\kappa\varepsilon^2}\text{log}\frac{1}{\varepsilon}\bigg)$ by setting $\varepsilon_\text{on}=\mathcal{O}(\varepsilon)$ and $\bar{\epsilon}^\text{adv}=\mathcal{O}(\varepsilon)$, where $c_1>0$ is a constant on offline setting.
\end{theorem}
\begin{proof}
    We have known that in the offline phase, the sample processed per cycle is $E|\mathcal{D}|$. As $\bar{\epsilon}^\text{adv}=\mathcal{O}(\frac{1}{\sqrt{E|\mathcal{D}|}})$, we can obtain $\bar{\epsilon}^\text{adv}\leq\frac{c_1}{\sqrt{E|\mathcal{D}|}}$. Therefore, it is immediately attained that $E\geq \frac{c^2_1}{(\bar{\epsilon}^\text{adv})^2|\mathcal{D}|}$. Hence, the per cycle offline sample complexity is $\mathcal{C}^\text{cycle}_\text{off}=\frac{c^2_1}{(\bar{\epsilon}^\text{adv})^2}$. Additionally, for the online phase, the online error is required to satisfy $\varepsilon_\text{on}=\tilde{\mathcal{O}}(\sqrt{\frac{H^3\varsigma}{T}})$ such that $\mathcal{C}^\text{cycle}_{\text{on}}=\tilde{\mathcal{O}}(\frac{H^3\varsigma}{\varepsilon_\text{on}^2})$. To ensure that $\Lambda=\tilde{\mathcal{O}}(\varepsilon)$, we let $\varepsilon_\text{on}=\mathcal{O}(\varepsilon)$ and $\bar{\epsilon}^\text{adv}=\mathcal{O}(\varepsilon)$, which lead to the following total sample complexity
    \begin{equation}
        \begin{split}\mathcal{C}_{\text{total}}&=K\mathcal{C}^\text{cycle}_\text{off}+K\mathcal{C}^\text{cycle}_{\text{on}}\\&=\bigg[\frac{1-\gamma}{\lambda\kappa}\text{log}\frac{2\Delta_0}{\varepsilon}\bigg]\bigg(\mathcal{O}\bigg(\frac{c^2_1}{\varepsilon^2}\bigg)+\tilde{\mathcal{O}}\bigg(\frac{H^3\varsigma}{\varepsilon^2}\bigg)\bigg).
        \end{split}
    \end{equation}
    The last equality is because of the cycle complexity obtained in Theorem~\ref{theorem_2}. The desirable result is immediately obtained by simplifying the bound.
\end{proof}
Theorem~\ref{theorem_3} reveals the total sample complexities over $K$ cycles of offline learning and online fine-tuning, \[\tilde{\mathcal{O}}\bigg(\underbrace{\frac{(1-\gamma)c^2_1}{\lambda\kappa\varepsilon^2}\text{log}\frac{1}{\varepsilon}}_{offline}+\underbrace{\frac{(1-\gamma)H^3\varsigma}{\lambda\kappa\varepsilon^2}\text{log}\frac{1}{\varepsilon}}_{online}\bigg).\] Surprisingly, the offline sample complexity is independent of $|\mathcal{D}|$ as the data reuse amortizes the cost. Both offline and online phases scale with $\tilde{\mathcal{O}}(\frac{1}{\varepsilon^2}\text{log}\frac{1}{\varepsilon})$, 
which matches the complexity bound in~\cite{xu2020improving}. 
The bound reveals that although KL regularization appears only in the offline phase, it affects sample complexity in both phases. A large $\lambda$ reduces cost but leads to conservative update and lower $\kappa$, while a small $\lambda$ accelerates learning but increases online sample use, highlighting a key trade-off between efficiency and learning performance.
From Table~\ref{table:complexity_com}, we know that pure online PPO has sample complexity of $\tilde{\mathcal{O}}(\frac{H^4\varsigma}{\varepsilon^2})$, COOPO's online phase reduces by factor $\frac{H\lambda\kappa}{(1-\gamma)\text{log}(1/\varepsilon)}$. Also, COOPO outperforms PPO in terms of horizon scaling with only $H^3$, which makes it suitable for long horizon ($H\gg 1$). This stems from the offline warm start. Theorem~\ref{theorem_3} also implies that online term dominates in low-$\varepsilon$ regime, consistent with empirical results.

\section{Numerical Results}
We evaluate COOPO on simulators to assess: (1) its performance against offline and offline-to-online RL baselines; (2) its asymptotic performance; (3) the role of offline–online cycles; (4) its ability to reduce online interaction compared to PPO; and (5) the impact of the multiplier $\lambda$ in offline learning. Experimental details are provided in the sequel.

\noindent\textbf{Benchmark anvironments.}
Figure~\ref{fig:mujoco_envs} shows the environments we resort to validate COOPO in this work, including Half-cheetah, Hopper, and Walker. While demonstrated in MuJoCo locomotion tasks, the cyclic offline-online paradigm generalizes to CPS domains such as adaptive robotic control, energy management, or manufacturing scheduling, where real-time policy refinement must go along with safety and latency constraints. Please see~\cite{todorov2012mujoco} for more information about these environments.

\noindent\textbf{Model architecture.} 
In this work, we follow the standard setup as widely used in RL domain to parameterize the actor and critic networks. Specifically, they are all multi-layer perceptron (MLP) models. The architecture is shown in Table~\ref{table:mlp_model}.

\begin{figure}
    \centering
    \includegraphics[width=1\linewidth]{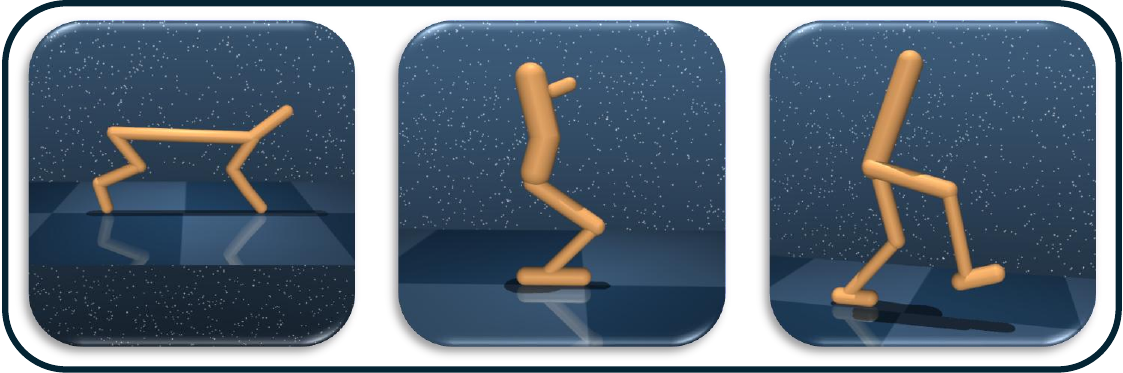}
    \caption{Mujoco environments used in this work for evaluation: Half-cheetah, Hopper, and Walker.}
    \label{fig:mujoco_envs}
\end{figure}

\begin{table}[h]
\caption{MLP model architecture}
\begin{center}
\begin{threeparttable}
\begin{tabular}{c c}
    \toprule
    \textbf{Parameter} & \textbf{Value} \\ \midrule
      \# of layers   &  4                     \\
      \# of hidden units per layer   &  256                                   \\ 
      Activation function & ReLU \\
      \bottomrule

\end{tabular}
\end{threeparttable}
\end{center}
\label{table:mlp_model}
\end{table}

\textbf{Hyperparameters.} The hyperparameter setting is in Table~\ref{table:hyper_params}. Note that in this context we summarize key hyperparameters in RL setting. Though a hyperparameter optimization method can likely be beneficial for the performance improvement, we tune them manually in this work.

\begin{table}[h]
\caption{Hyperparameters}
\begin{center}
\begin{threeparttable}
\begin{tabular}{c c}
    \toprule
    \textbf{Hyperparameter.} & \textbf{Value} \\ \midrule
      Optimizer   &  Adam                     \\
      Learning rate   &  3e-4                                   \\ 
     Batch size (off) &  512\\
     Batch size (on) &  64\\
     \# of epochs (off) & 100 \\
     size of dataset (off) & 1M \\
     \# of epochs (on) & 5 \\
     \# of episodes (on) & 5\\
     \# of cycles & 500\\
     Initialization & Random\\
     Episode length (on) & 1M \\
     buffer size (on) & 512\\
     discount factor & 0.99\\
     beta & 0.99\\
     KL weight & 0.05\\

      \bottomrule
\end{tabular}
\end{threeparttable}
\end{center}
\label{table:hyper_params}
\vspace{-0.1in}
\end{table}

\noindent\textbf{Computing Infrastructure.}
All experiments were conducted on a workstation equipped with an NVIDIA GeForce RTX\,4090 GPU and an Intel(R) Core(TM)\,i7‑14700 CPU, with each run utilizing approximately 1.4\,GiB of GPU memory. The system was running Ubuntu\,22.04.4\,LTS\,(64‑bit) (Distributor\,ID:\,Ubuntu, Release:\,22.04). Offline datasets were obtained from the D4RL benchmark suite, while online training was performed using Gym environments (version\,0.23.1). The implementation leveraged standard deep learning and reinforcement learning frameworks, including PyTorch (version\,2.4.1), and other supporting libraries such as NumPy and Matplotlib.

\begin{table*}[h!]
\caption{Reward values of various algorithms on the D4RL locomotion benchmarks using medium, medium-replay, and medium-expert datasets. Each value is the mean return over seeds, averaged across the final evaluation episodes.}
\centering
\resizebox{\textwidth}{!}{%
\begin{tabular}{l
>{\columncolor{blue!30}}c 
>{\columncolor{blue!30}}c 
>{\columncolor{blue!30}}c 
>{\columncolor{blue!30}}c 
>{\columncolor{blue!30}}c 
>{\columncolor{green!15}}c 
>{\columncolor{green!15}}c 
>{\columncolor{green!15}}c 
>{\columncolor{green!15}}c 
>{\columncolor{green!15}}c 
>{\columncolor{green!15}}c 
>{\columncolor{orange!20}}c 
}
\toprule
\textbf{Dataset} &
\textbf{BC} &
\textbf{Onestep RL} &
\textbf{TD3+BC} &
\textbf{CQL} &
\textbf{IQL} &
\textbf{Cal-QL} &
\textbf{AWAC} &
\textbf{Uni-O4} &
\textbf{ODT} &
\textbf{PEX} &
\textbf{Off2ON} &
\textbf{COOPO (Ours)} \\
\midrule
\rowcolor{cyan!10}
halfcheetah-medium        & 4233.3 & 4828.4 & 4822.7 & 4371.2 & 4520.6 & 4754.7 & 4531.1 & \textbf{5277.5} & 4235.4 & 5084.2 & 3878.4 & 4947.1 \\
\rowcolor{pink!10}
hopper-medium             &  885.8 &  995.6 &  990.6 &  977.7 & 1092.3 & 1606.1 &  967.3 & 1740.8 &  781.7 &  937.8 & 1625.0 & \textbf{2014.7} \\
\rowcolor{yellow!10}
walker2d-medium           & 2748.3 & 2987.6 & 3052.0 & 2651.3 & 2718.4 & 3083.6 & 2910.2 & 3289.3 & 2625.7 & 2922.4 & 2413.6 & \textbf{3570.4} \\
\rowcolor{cyan!10}
halfcheetah-medium-replay & 1967.7 & 2053.5 & 2423.5 & 2477.9 & 3368.4 & 1795.2 & 2189.6 & 2911.2 & 2087.6 & 2995.8 & 2747.4 & \textbf{4602.5} \\
\rowcolor{pink!10}
hopper-medium-replay      &  308.5 & 1620.1 & 1016.0 & 1580.7 &  913.2 &  714.8 &  624.2 & 1197.7 &  698.8 &  360.3 &  322.1 & \textbf{1993.8} \\
\rowcolor{yellow!10}
walker2d-medium-replay    & 1037.3 & 1972.6 & 3252.7 & 3073.5 & 2076.4 & 2271.8 & 1075.8 & 3072.3 & 1380.2 & 2503.7 &  553.6 & \textbf{3491.2} \\
\rowcolor{cyan!10}
halfcheetah-medium-expert & 5192.7 & 8949.3 & 8674.5 & 8762.8 & 8278.3 & 8782.9 & 3976.5 & 8859.3 & 3566.1 & 2778.3 & 4333.0 & \textbf{9242.2} \\
\rowcolor{pink!10}
hopper-medium-expert      & 1269.7 & 2484.6 & 2359.7 & 2534.6 & 2201.3 & 2596.7 & 1339.3 & 2683.2 & 1556.3 & 2196.7 & 2697.0 & \textbf{3575.8} \\
\rowcolor{yellow!10}
walker2d-medium-expert    & 4936.6 & 5189.1 & 5056.0 & 4996.3 & 5033.0 & 5102.7 & 4643.2 & \textbf{5421.3} & 3040.0 & 3205.1 & 3934.8 & 5042.8 \\
\bottomrule
\end{tabular}%
}
\label{table:performance_comparison}
\end{table*}

\noindent\textbf{Comparative study.}
We compare \textcolor{orange!80!black}{\textbf{COOPO}} with multiple offline and offline-to-online algorithms. \textbf{\textcolor{blue!83!black}{For offline RL}}, the methods encompass Behavior Cloning (BC) ~\cite{torabi2018behavioral}, Onestep RL ~\cite{eysenbach2023connection}, TD3+BC~\cite{fujimoto2021minimalist}, Conservative Q-Learning (CQL)~\cite{kumar2020conservative} and Implicit Q-Learning (IQL)~\cite{kostrikov2021offline}. \textbf{\textcolor{green!60!black}{For offline-to-online RL}}, we also include multiple popular baseline methods that have been adopted in many works: Cal-QL~\cite{nakamoto2023cal}, Advantage Weighted Actor-Critic (AWAC)~\cite{nair2020awac}, Uni-O4~\cite{lei2023uni}, Online Decision Transformer (ODT)~\cite{zheng2022online}, PEX~\cite{zhang2023policy}, and Off2ON~\cite{lee2022offline}. 
We first evaluate whether COOPO outperforms or remains competitive with offline and offline-to-online approaches. As shown in Table~\ref{table:performance_comparison} COOPO achieves top performance on 7 out of 9 tasks, surpassing all one-step methods like IQL and Onestep RL, which often yield conservative policies due to behavior constraints. COOPO also outperforms pure offline RL methods, highlighting the benefit of online fine-tuning for effective exploration. Compared to offline-to-online baselines, COOPO consistently excels across most environments, remaining competitive with Uni-O4 on select tasks. These gains stem from its cyclic synergy between offline and online learning, which mitigates distributional shift and catastrophic forgetting. Figure~\ref{fig:asympt_curve} further confirms COOPO's superior asymptotic performance demonstrated in HalfCheetah compared to AWAC and IQL, aligning with trends in Figure~\ref{fig:coopo_curve} and emphasizing the benefits of data reuse and stable policy refinement.

\begin{figure}[h]
    \centering
    \includegraphics[width=0.8\linewidth]{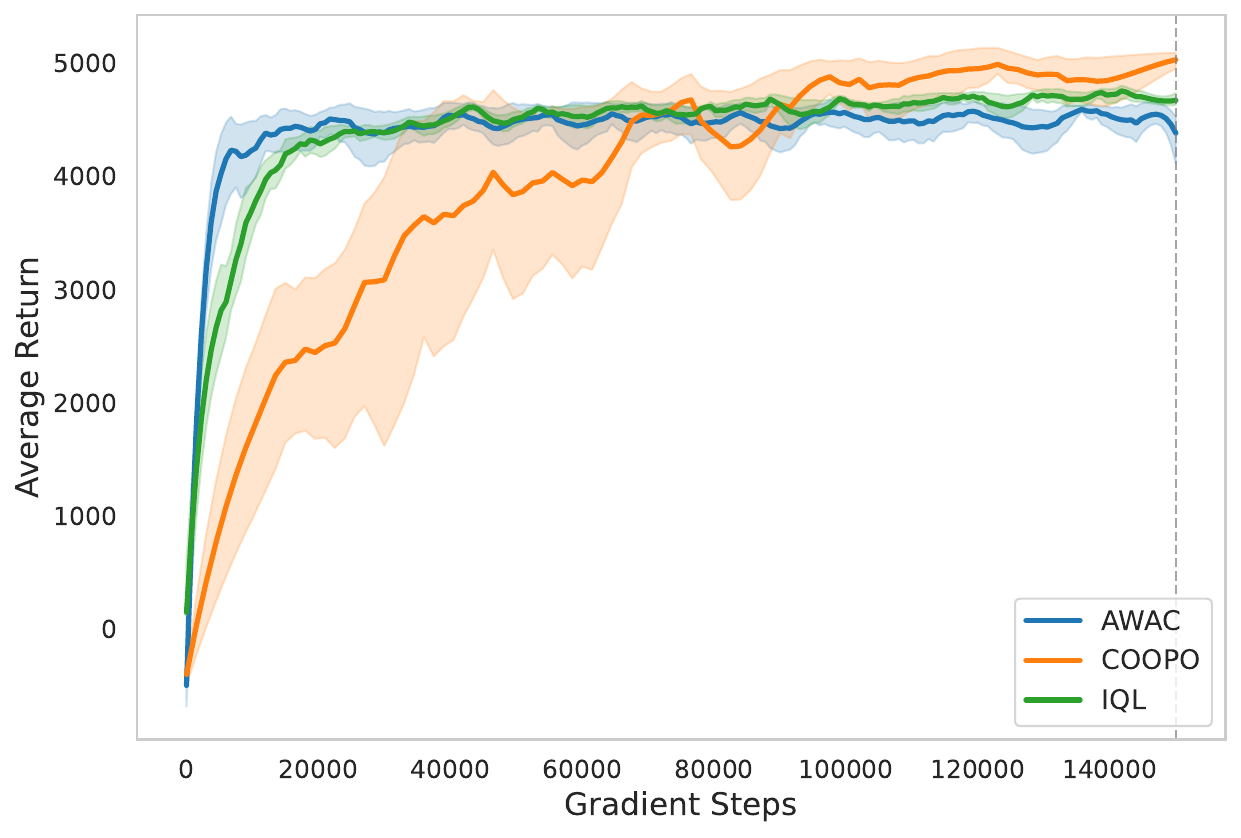}
    \caption{Average return vs. gradient steps on HalfCheetah for AWAC, COOPO, and IQL.  COOPO achieves the best final performance with improved sample efficiency, while AWAC and IQL converge faster but plateau earlier.}
    \label{fig:asympt_curve}
\end{figure}

\noindent\textbf{Offline and Online Training Phases in COOPO.}
Figure~\ref{fig:offline-online} illustrates the training dynamics of our proposed method COOPO, which alternates between offline and online training phases, on the HalfCheetah, Hopper, and Walker2D tasks using the D4RL medium datasets. In each plot, the red{red solid lines} represent online training segments, while the blue{blue dashed lines} indicate offline updates between online training phases. Across all three environments, the offline updates consistently improve policy performance. Although there may be rare cases where offline updates momentarily hinder online training, overall they help the online phase achieve higher rewards. The results highlight the importance of COOPO's alternating architecture between offline and online training phases to achieve better performance and more efficient learning.

\begin{figure}[t]
  \centering
  \begin{subfigure}[b]{0.48\linewidth}
    \centering
    \includegraphics[width=\linewidth]{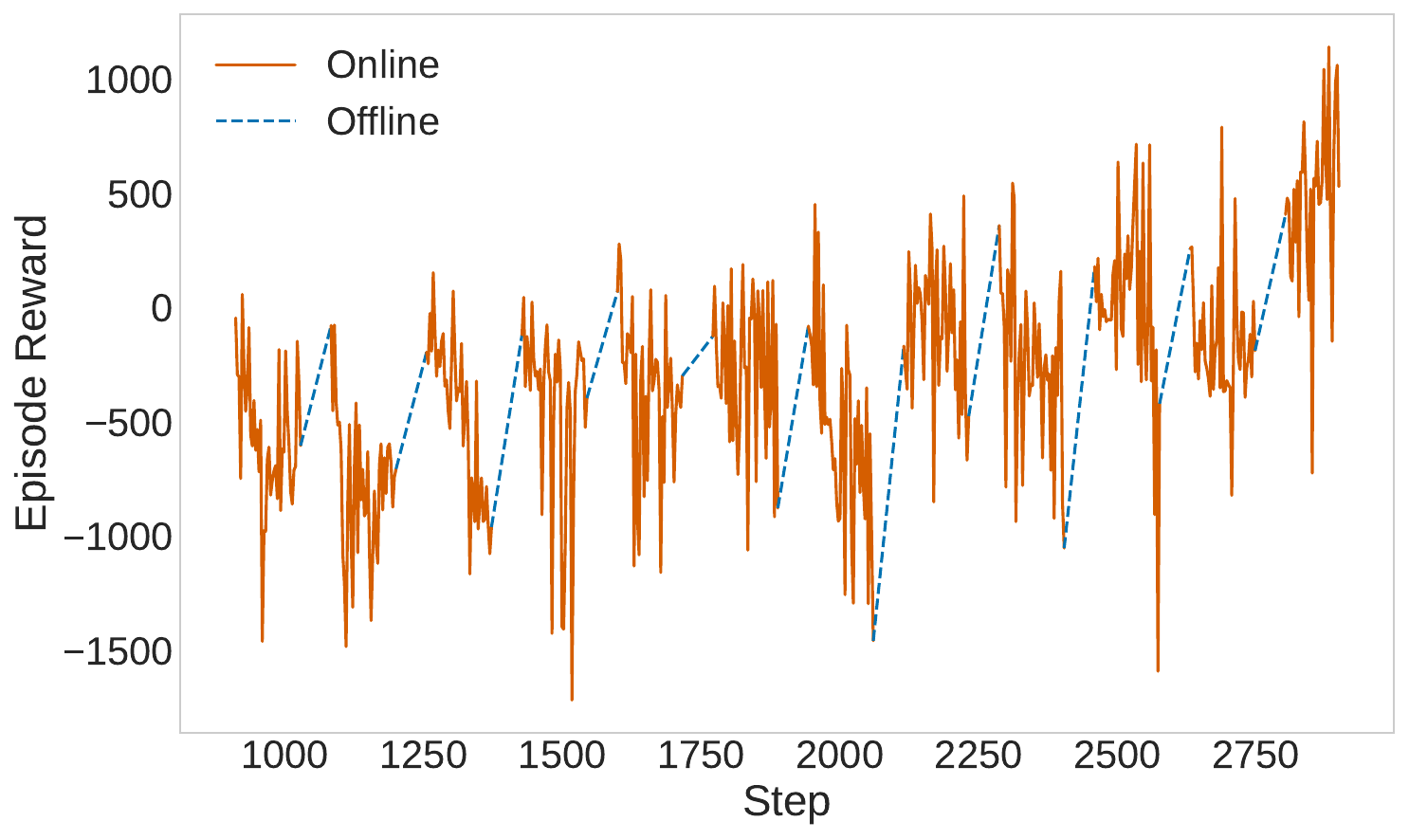}
    \caption{HalfCheetah}
    \label{fig:halfcheetah}
  \end{subfigure}\hfill
  \begin{subfigure}[b]{0.48\linewidth}
    \centering
    \includegraphics[width=\linewidth]{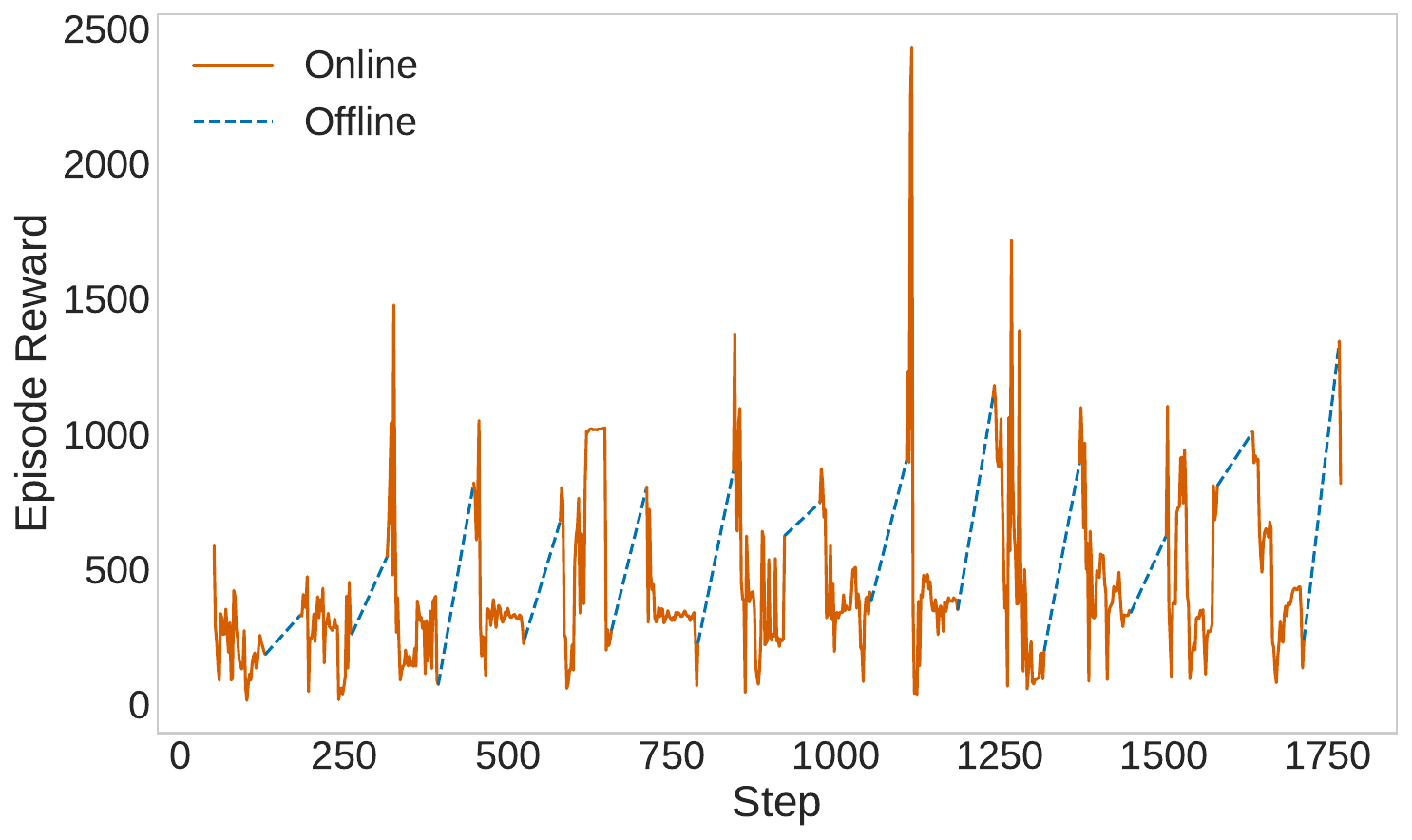}
    \caption{Hopper}
    \label{fig:hopper}
  \end{subfigure}

  \medskip 

  \hspace*{\fill}
  \begin{subfigure}[b]{0.48\linewidth}
    \centering
    \includegraphics[width=\linewidth]{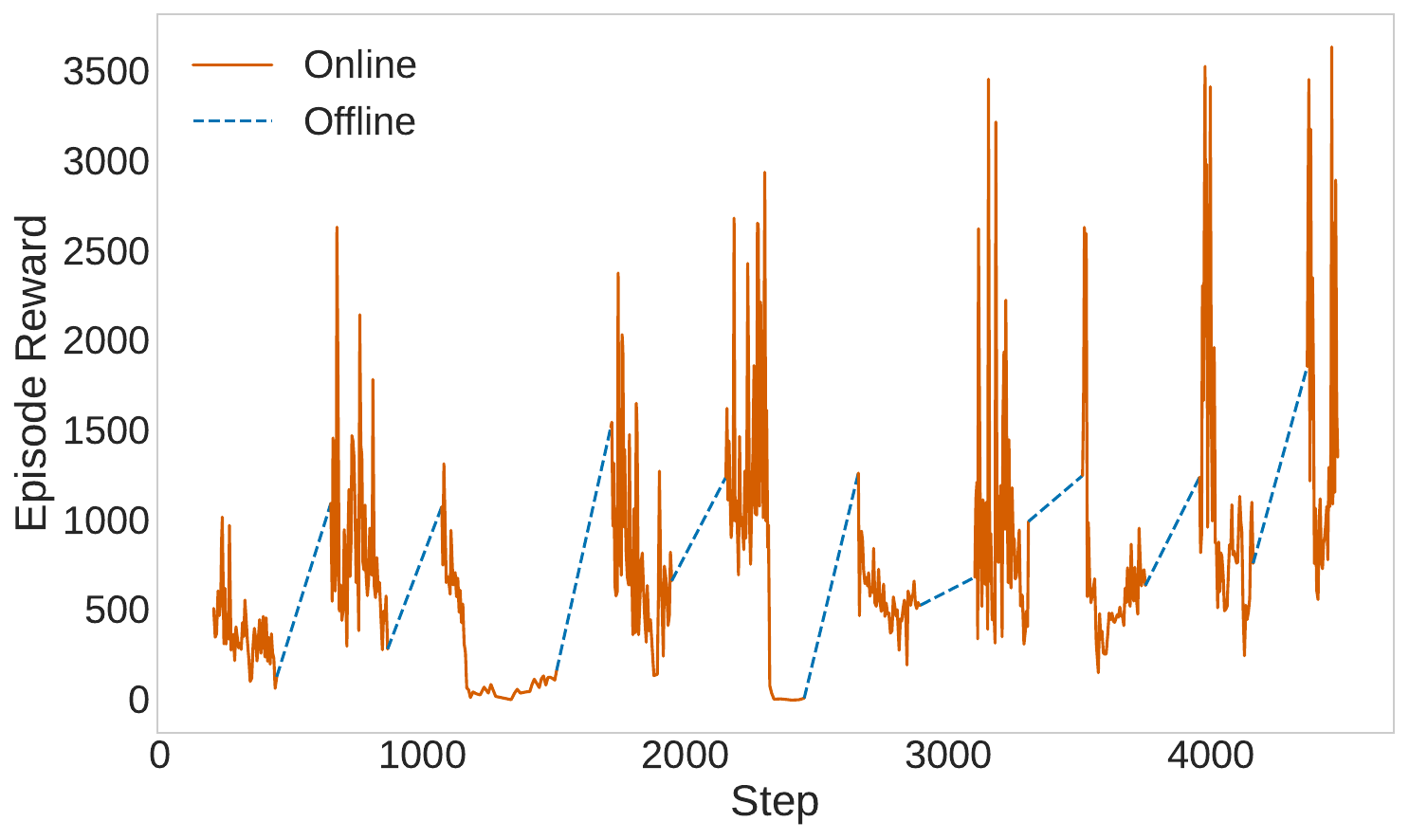}
    \caption{Walker2D}
    \label{fig:walker}
  \end{subfigure}
  \hspace*{\fill}

  \caption{Illustration of the training dynamics of COOPO on three D4RL medium datasets. 
  The red solid lines represent the online training segments, while the blue dashed lines indicate the offline updates bridging the gaps between online phases. 
  For clarity, the figure is using a single training seed from each environment.}
  \label{fig:offline-online}
  \vspace{-0.15in}
\end{figure}

\noindent\textbf{Ablation studies.}
We now examine the interplay between offline and online training phases in each cycle. As shown in Figure~\ref{fig:offline-online-cycles} and outlined in Algorithm~\ref{alg:coopo}, each COOPO cycle uses a small, fixed number of online and offline rollout episodes (T = 5 to 25, E = 25 to 100), depending on the task. This bounded sampling budget reduces total environment interactions by approximately 50\% compared to PPO while still allowing meaningful policy refinement.

We evaluate performance across various combinations of offline epochs ($E$) and online fine-tuning steps ($T$). When the offline training epochs are fixed at 100, increasing the frequency of online fine-tuning (e.g, $T=5$ vs. $T=25$) improves performance, particularly in the low-$\varepsilon$ regime near optimality, supporting the theoretical insights from Theorem~\ref{theorem_3}. This highlights the greater impact of online updates in the low-$\varepsilon$ regime, as predicted by the sample complexity analysis. Conversely, when the frequency of online fine-tuning is fixed, increasing offline training epochs boosts performance by anchoring the policy more effectively to the static dataset, thus mitigating distributional shift and catastrophic forgetting. Interestingly, when offline epochs increase (e.g., from $E=75$ to $E=100$) while online episodes decrease (e.g., $T=15$ to $T=5$), performance improves early in training but plateaus later, indicating diminishing returns from offline updates alone. In contrast, increasing online fine-tuning leads to more sustained gains, as seen when comparing $T=20, E=50$ with $T=5, E=100$. These findings demonstrate COOPO’s capacity to reduce online environment interaction by reusing offline data. However, if offline training is too sparse, even frequent online fine-tuning may not be sufficient, which is evident when comparing $T=20, E=25$ with $T=5, E=100$. Therefore, to deploy COOPO effectively, it is beneficial to emphasize offline training in the early (high-$\varepsilon$) phase and prioritize online fine-tuning in the later (low-$\varepsilon$) phase.

Figure~\ref{fig:coopo_vs_ppo} compares the number of trajectories required by COOPO and PPO to achieve high performance. COOPO consistently attains higher rewards using significantly fewer trajectories, highlighting its sample efficiency. In contrast, PPO suffers from high sample complexity due to its on-policy nature. These results underscore the effectiveness of COOPO’s offline data reuse and the benefits of the proposed looped synergy. Moreover, they empirically validate the theoretical claim from Theorem~\ref{theorem_3} that COOPO reduces the required online interactions by a constant factor. We now answer the last question of how the key $\lambda$ value in the offline learning affects the overall performance. 

\begin{figure}[h]
    \centering
    \includegraphics[width=0.8\linewidth]{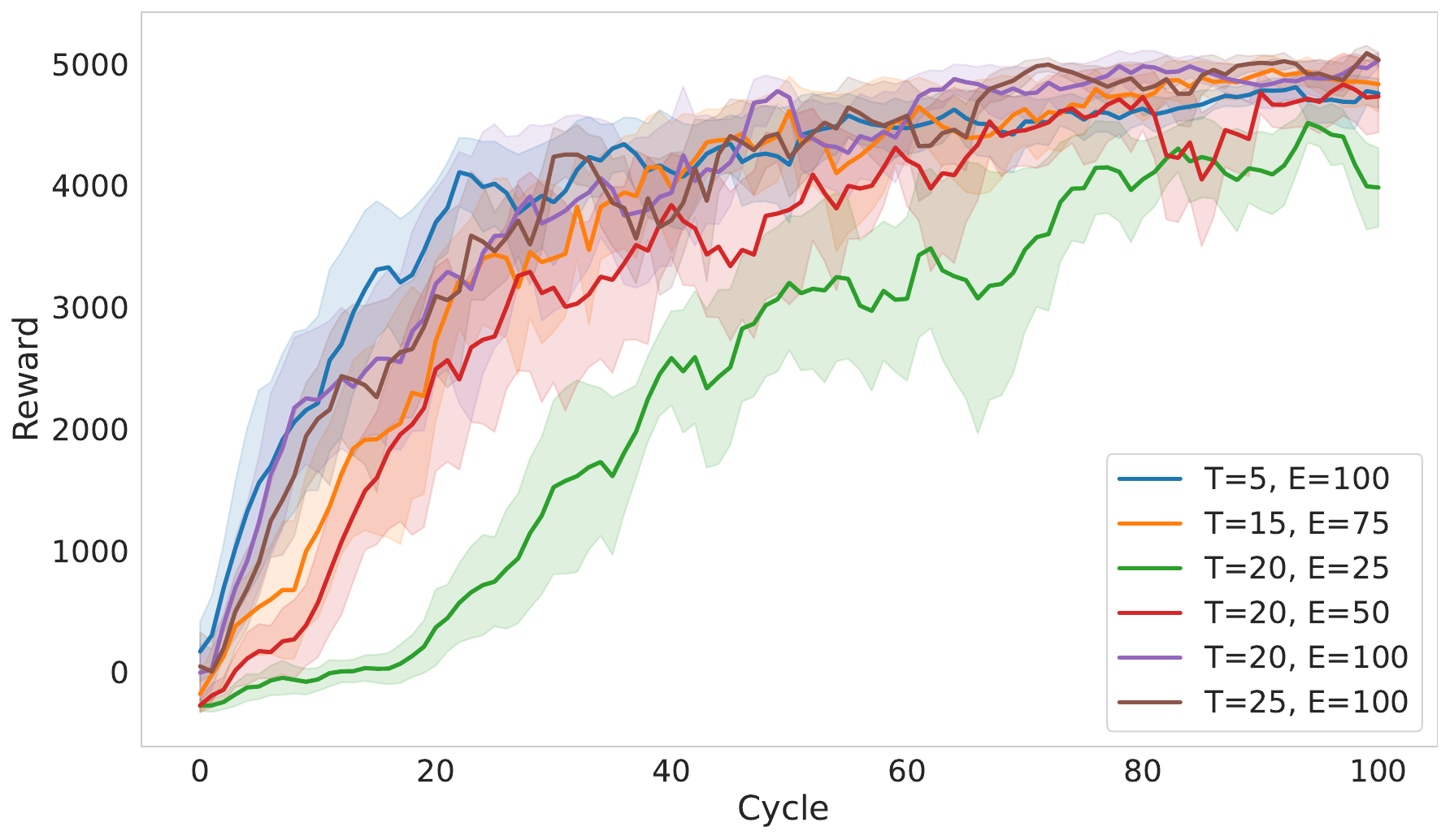}
    \caption{
        Training performance across different online-offline cycle combinations on the \textit{HalfCheetah} environment.
        Increasing offline cycles generally improves stability and reduces the need for extensive online updates.
    }
    \label{fig:offline-online-cycles}
\end{figure}

\begin{figure}[h]
    \centering
    \includegraphics[width=0.8\linewidth]{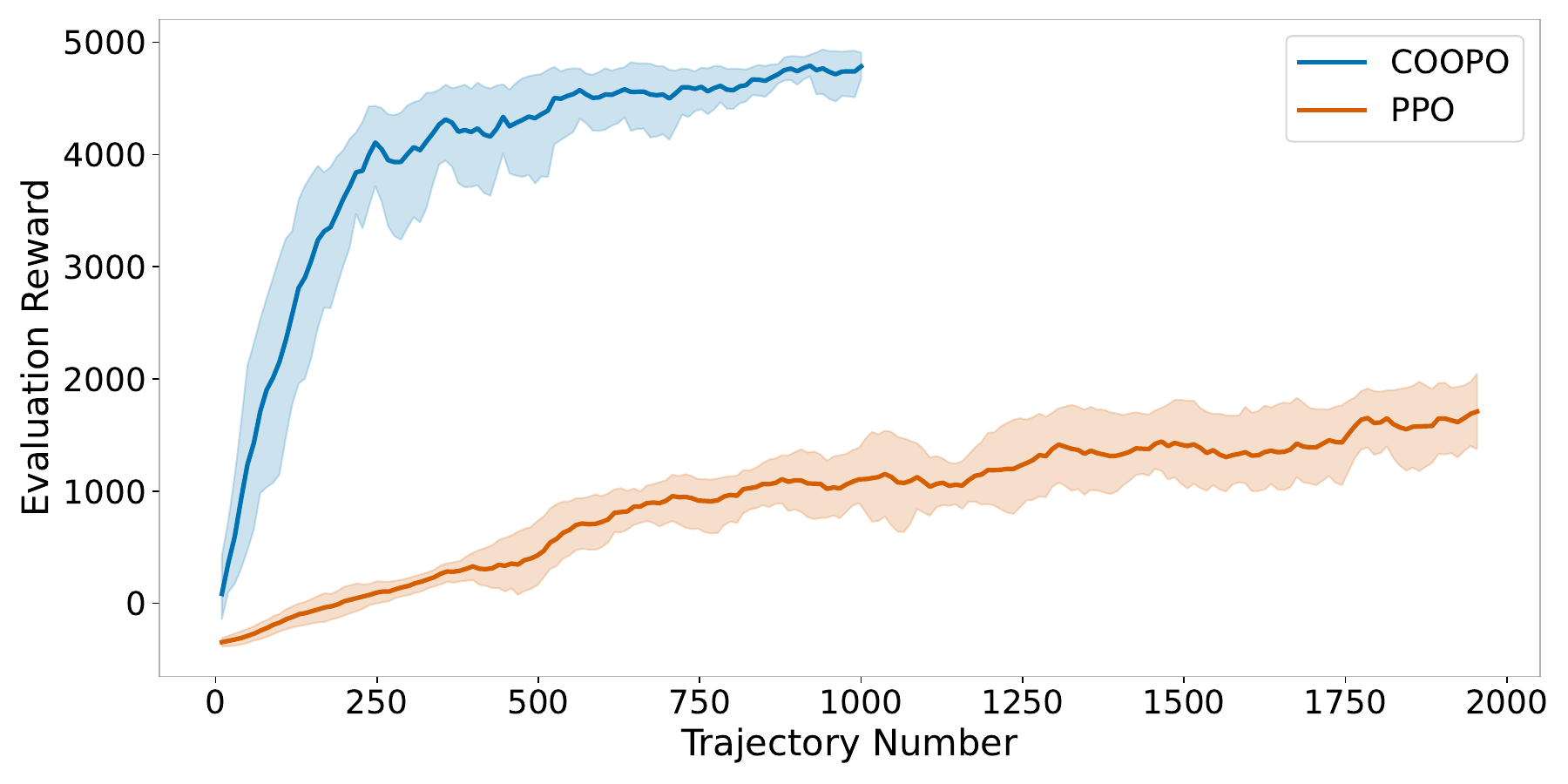}
    \caption{
        Training performance against trajectory number between COOPO vs. PPO on the \textit{HalfCheetah} environment.
    }
    \label{fig:coopo_vs_ppo}
\end{figure}
\begin{figure}[h]
  \centering
  \begin{subfigure}[t]{0.48\linewidth}
    \centering
    \includegraphics[width=\linewidth]{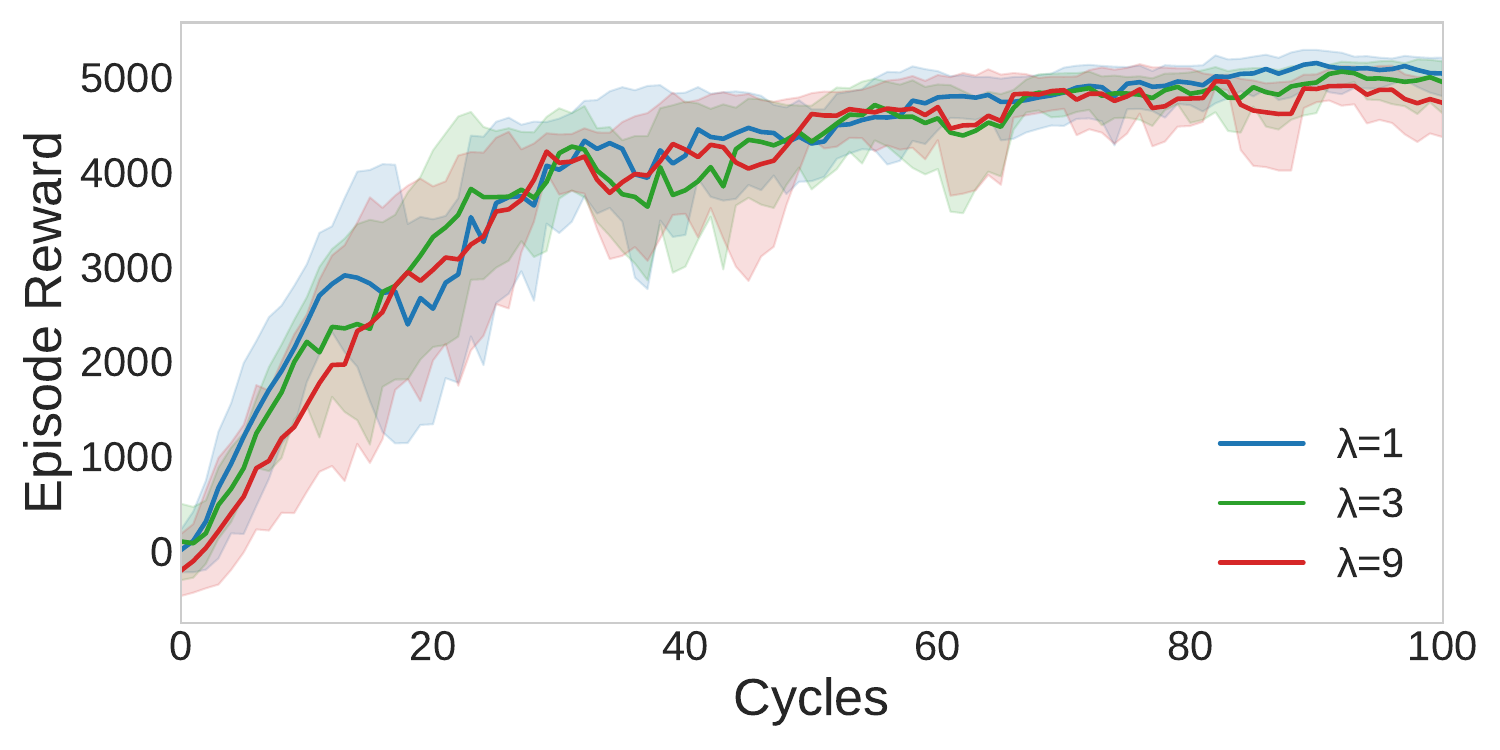}
    \caption{Training curves for different $\lambda$ values. Shaded regions show reward ranges across seeds.}
    \label{fig:coopo-line}
  \end{subfigure}
  \hfill
  \begin{subfigure}[t]{0.48\linewidth}
    \centering
    \includegraphics[width=\linewidth]{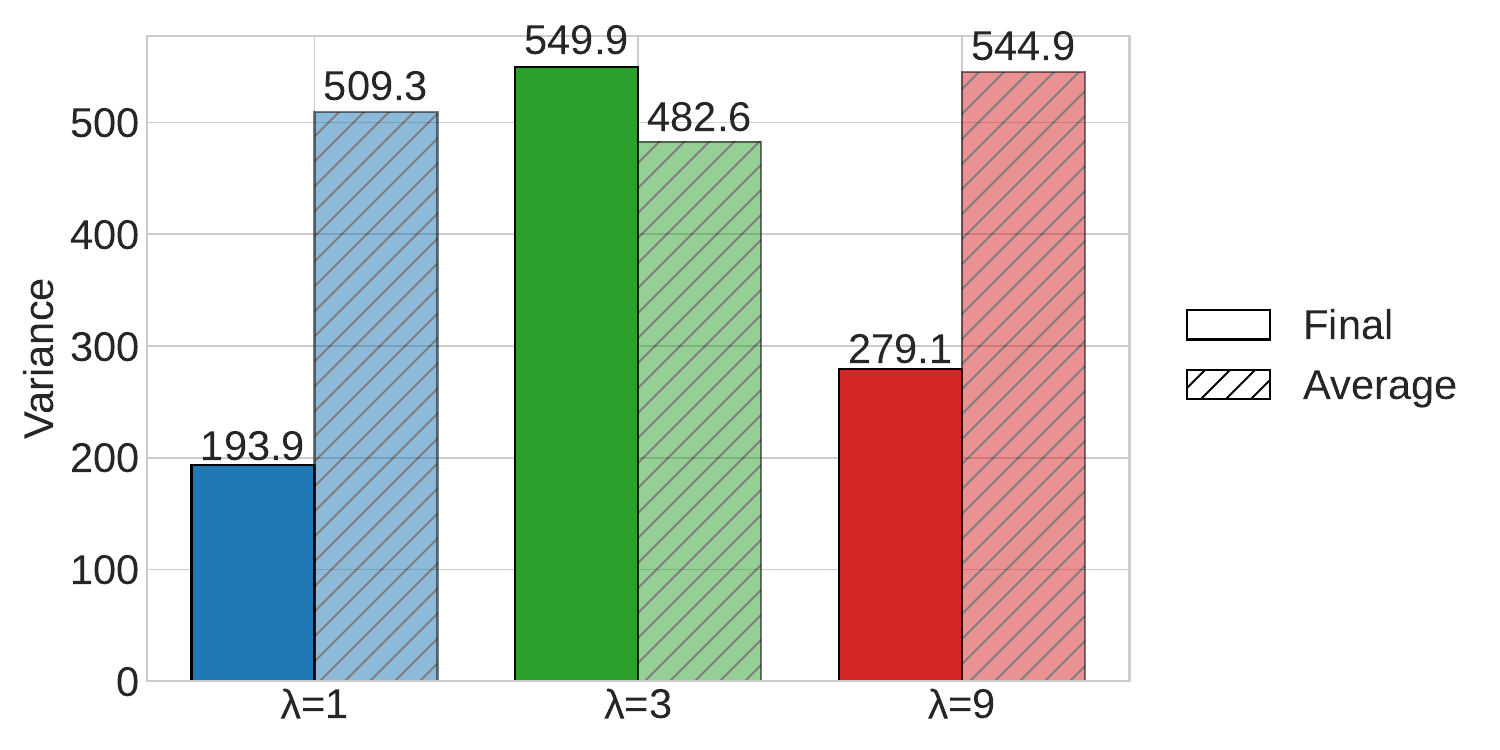}
    \caption{Final reward variance vs. average variance across seeds for each $\lambda$.}
    \label{fig:coopo-bar}
  \end{subfigure}
  \caption{Performance comparison in the \textit{HalfCheetah} environment for different $\lambda$ values, averaged over multiple seeds.}
  \label{fig:coopo_lambda}
  \vspace{-0.15in}
\end{figure}

Figure~\ref{fig:coopo_lambda} (a) shows that all $\lambda$ values eventually reach high episode rewards, but smaller $\lambda$ (e.g., $\lambda=1$) leads to faster convergence, while larger (e.g., $\lambda=9$) results in more conservative updates. Figure~\ref{fig:coopo_lambda} (b) reveals that $\lambda=3$ strikes a balance by achieving relatively fast convergence with the lowest final and average variance across seeds, indicating more stable and consistent policy updates compared to the more variable results of the other two cases. 

\noindent\textbf{Limitations.} While COOPO enhances sample efficiency, its effectiveness is constrained by the quality and coverage of the offline dataset poorly curated data may trap policies in local optima, limiting online gains. The approach also struggles with sparse-reward environments due to limited exploration during brief online phases, hindering novel discovery. Additionally, computational overhead from repeated offline retraining increases wall-clock time, and static datasets cannot dynamically incorporate online experience, capping adaptability in non-stationary tasks.

\section{Conclusions}
COOPO introduces a cyclic synergy between offline and online reinforcement learning, obtaining stable policy improvement under limited interaction budgets. 
By alternating KL-regularized offline optimization with trust-region online fine-tuning, the framework mitigates distributional drift, prevents catastrophic forgetting, and provides a stabilizing corrective mechanism when online exploration becomes trapped or noisy. 
Theoretically, COOPO offers guarantees on bounded reduced online interaction and a performance lower bound, which provides a principled foundation for reliable deployment in cyber-physical systems. Empirically, COOPO achieves consistently strong returns and sample efficiency across diverse benchmarks, outperforming recent state-of-the-art offline-to-online RL approaches. 
These results highlight COOPO as a practical and robust solution for safe and adaptive learning in real-world CPS.

\bibliographystyle{IEEEtran}
\bibliography{COOPO}

\end{document}